\documentclass[11pt, a4paper, copyright, google]{google}

\usepackage[authoryear, sort&compress, round]{natbib}
\usepackage{nicefrac}
\usepackage{float}
\usepackage{algorithm}
\usepackage{algorithmic}
\usepackage{multirow}
\usepackage{subcaption}
\usepackage{wrapfig}
\usepackage{xspace}
\usepackage[most]{tcolorbox}

\newtheorem{theorem}{Theorem}
\definecolor{promptbgcolor}{RGB}{245,245,245} 
\definecolor{promptframecolor}{RGB}{180,180,180} 

\newtcolorbox{promptbox}[1][]{
  enhanced,
  boxrule=0.5pt,
  colback=promptbgcolor,
  colframe=promptframecolor,
  fonttitle=\bfseries,
  coltitle=black,
  attach boxed title to top left={yshift=-2mm, xshift=2mm},
  boxed title style={
    boxrule=0pt,
    colframe=white,
    colback=white,
  },
  title={#1},
  arc=2mm, 
  breakable, 
}

\newcommand{\method}{PACEvolve++}

\uselogo{}

\title{\method{}: Improving Test-time Learning for Evolutionary Search Agents}

\correspondingauthor{myan46@wisc.edu}

\renewcommand{\copyrightext}{\footerfont $^*$Work done while at Google.}

\author[1,2*]{Minghao Yan}
\author[3*]{Bo Peng}
\author[4]{Benjamin Coleman}
\author[1]{Ziqi Chen}
\author[4]{Zhouhang Xie}
\author[4]{Shuo Chen}
\author[4]{Zhankui He}
\author[4]{Noveen Sachdeva}
\author[1]{Weili Wang}
\author[4]{Ed H. Chi}
\author[2]{Shivaram Venkataraman}
\author[4]{Wang-Cheng Kang}
\author[4]{Derek Zhiyuan Cheng}
\author[1]{Beidou Wang}
\affil[1]{Google}
\affil[2]{University of Wisconsin--Madison}
\affil[3]{Meta}
\affil[4]{Google DeepMind}

\begin{abstract}
Large language models have become drivers of evolutionary search, but most systems rely on a fixed, prompt-elicited policy to sample next candidates. This limits adaptation in practical engineering and research tasks, where evaluations are expensive, and progress depends on learning task-specific search dynamics. We introduce \method{}, an advisor-model reinforcement learning framework for test-time policy adaptation in evolutionary search agents. \method{} decouples strategic search decisions from implementation: a trainable advisor generates, assesses, and selects hypotheses, while a stronger frontier model translates selected hypotheses into executable candidates. To train the advisor under non-stationary feedback, we propose a phase-adaptive approach that adapts its optimization strategy to different phases of the evolutionary process. Early in evolution, it uses group-relative feedback to learn broad search preferences; later, as reward gaps compress, it emphasizes best-of-$k$ frontier contribution to support stable refinement. Across expert-parallel load balancing, sequential recommendation, and protein fitness extrapolation, \method{} outperforms the state-of-the-art evolutionary search framework with frontier models, achieving faster convergence and stabilizing test-time training during evolutionary search.
\end{abstract}

\begin{document}

\maketitle

\section{Introduction}

Large language models (LLMs) have recently emerged as effective drivers of evolutionary program search, enabling autonomous discovery for open-ended optimization problems~\citep{funsearch, shinkaevolve, openevolve}. In this paradigm, an agent repeatedly inspects the current best solution, its evaluation metrics, and the search history, proposes candidate mutations, and retains the best-performing descendant. This simple loop has proved remarkably effective: AlphaEvolve~\citep{alphaevolve} demonstrated state-of-the-art algorithm discovery in domains such as bin packing, matrix multiplication, and circle packing, while subsequent open-weight systems extended these gains to symbolic regression and kernel optimization~\citep{shojaee2024llm, shinkaevolve}. More recent systems improve the external mechanics of this loop through stronger context management, backtracking, population maintenance, and self-adaptive workflows~\citep{yan2026pacevolve, cemri2026adaevolve, liu2026evox}. These advances make long-horizon search substantially more reliable. Still, they typically rely on a fixed-parameter, prompt-elicited reasoning policy: useful search experience may accumulate in the scaffold, but it is not directly internalized into the model's decision preferences. This leaves a central question open: \emph{how should we adapt an LLM's reasoning policy to make better search decisions during long-horizon evolutionary optimization?}

This need for policy adaptation becomes especially consequential in practical research and engineering tasks~\citep{yang2025reinforcement, chan2024mle, qiang2025mle}. In these domains, effective search decisions often depend on recognizing patterns across previous attempts: which mutation families repeatedly fail, which partial improvements are worth revisiting, and which directions remain novel relative to the evolving frontier. In recommender-system design~\citep{fuxictr, fuxictr2}, MoE load balancing~\citep{agrawal2025gepa, liu2024deepseek}, and protein fitness extrapolation~\citep{tran2026rapid}, candidate directions may range from architectural changes and optimization choices to routing strategies~\citep{cheng2025barbarians}, feature interactions~\citep{wang2021dcn}, and sequence-level transformations~\citep{guo2017deepfm}. Many such directions can be justified by generic LLM reasoning, but only a few produce measurable improvement after evaluation~\citep{liu2025fitness}. A fixed policy can condition on this history through context, but it does not internalize the resulting search feedback into stable decision preferences~\citep{anthony2025language, zhu2025failure, nie2024evolve}. Thus, the key challenge is not merely generating plausible hypotheses, but adapting the model's decision policy to prioritize directions that are novel, feasible, and likely to improve the evolving frontier.

\begin{wrapfigure}{r}{0.5\linewidth}
    \centering
    \vspace{-10pt}
    \includegraphics[width=\linewidth]{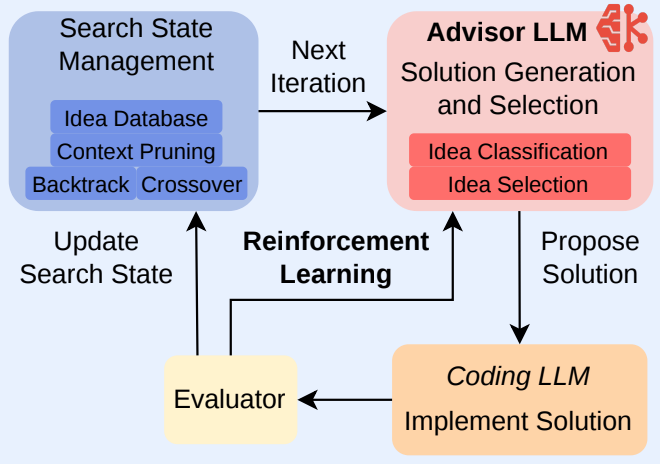}
    \vspace{-5pt}
    \caption{Overall \method{} workflow. A trainable advisor handles idea generation, novelty assessment, and hypothesis selection, while a frontier implementation model writes code. The RL objective is coupled to rollout batches and adapts its credit assignment to the search phase.}
    \vspace{-5pt}
    \label{fig:pacerl}
\end{wrapfigure}

We introduce a dedicated advisor model~\citep{asawa2025train} to make search-specific policy adaptation explicit. The advisor learns the strategic decisions in evolutionary search~\citep{yan2026pacevolve}, such as hypothesis generation, novelty assessment, and mutation selection, while a stronger frontier implementation model translates the selected hypothesis into executable code~\citep{gemini3}. This design departs from standard evolutionary coding frameworks, which often use the same model to both decide what to try and implement the resulting mutation~\citep{wang2025thetaevolve, yuksekgonul2026learning}. Such coupling can be suboptimal in practical research and engineering tasks, where implementation failures can arise from complex codebases, integration details, and system constraints~\citep{yang2026programbench}. In these settings, the search-specific signal lies primarily in deciding which hypothesis is novel, feasible, and likely to improve the evolving frontier, separate from the model's general coding capabilities~\citep{team2026kimi, zeng2026glm}. End-to-end training, therefore, entangles hypothesis quality with implementation correctness, making them noisy signals for adapting search preferences. By isolating the advisor as the trainable decision layer, our framework focuses reinforcement learning on what to evaluate next while leveraging frontier coding models for implementation.

With the advisor model paradigm (\S~\ref{sec:advisor}), the remaining challenge is to learn from feedback whose usefulness evolves over time. Early in search, the policy should be encouraged to explore broad search directions: candidates often differ substantially in mechanism and quality, and group-relative feedback provides an informative signal for learning which mutation families are promising~\citep{liu2025fitness}. Later, however, the search increasingly mutates already strong descendants~\citep{lengler2019drift}, resulting in marginal differences between candidates that make group-relative signals ineffective. We address this with \emph{phase-adaptive RL} (\S~\ref{sec:hybrid}). During early exploration, we aim to incentivize the advisor to identify useful search directions from diverse candidates without prematurely collapsing onto a few high-scoring rollouts (Figure~\ref{fig:main_results}). As search moves toward refinement and reward gaps compress, the objective gradually shifts toward frontier-contribution feedback and assigning credit based on whether a candidate contributes to the evolving best-of-$k$ frontier. This late-stage signal does not simply imitate the highest-scoring rollout; it credits candidates based on their contributions to frontier improvement~\citep{walder2025pass}. The resulting recipe aligns training with the log-diminishing reward structure of evolutionary search dynamics~\citep{cheng2025barbarians}, stabilizing late-stage training while avoiding early-stage exploitation (Theorem~\ref{thm:reward-collapse}), enabling the policy first to learn broad search preferences and then focus on high-value refinements near the frontier (\S~\ref{sec:exp}). 
In summary, we introduce an advisor-model reinforcement learning framework (\S~\ref{sec:workflow}) for self-evolving agents. Our contributions include:
\begin{itemize}[leftmargin=0.5em]
\item We design an advisor-based policy adaptation (\S~\ref{sec:advisor}), where we decouple search-decision learning from code implementation by training an advisor for hypothesis generation, novelty assessment, and mutation selection, while delegating executable-code realization to a stronger frontier implementation model. 
\item We design a search-dynamics-aware reinforcement learning algorithm (\S~\ref{sec:hybrid}) based on this framework. We develop a phase-adaptive recipe that shifts credit assignment from group-relative feedback during exploration to frontier-contribution during refinement, aligning policy learning with evolutionary search dynamics. 
\item Empirically, we demonstrate strong performance across a range of real-world research and engineering tasks (\S~\ref{sec:tasks}), including expert-parallel load balancing~\citep{deepseekai2026deepseekv4}, sequential recommendation~\citep{ye2026fuxi}, and protein fitness extrapolation~\citep{tran2026rapid}, outperforming while converging faster than existing methods with and without RL (\S~\ref{sec:exp}).
\end{itemize}

\section{Background}

\subsection{Evolutionary Search Agents}

An evolutionary search agent improves a program through repeated proposal, evaluation, and selection~\citep{holland1992genetic, fogel1988evolutionary, hornby2006automated}. Given an initial program $p_0$, an evaluator $\mathcal{E}: \mathcal{P} \rightarrow \mathbb{R}$, and a policy $\pi_\theta$, the agent generates candidate modifications, evaluates them, and updates the current solution whenever a higher-scoring descendant is identified. At iteration $t$, the policy conditions on (one of) the current best programs $p_t$, their evaluation metrics, and the accumulated search history to generate candidates $\{p_t^{(1)}, \ldots, p_t^{(n)}\}$. If the candidate scores high, it is then added to a set of the best candidate programs for future reference.

This line of work has progressed along two complementary directions. The first improves the \emph{search scaffold}. FunSearch~\citep{funsearch} and AlphaEvolve~\citep{alphaevolve} showed that strong results can emerge from repeated in-context mutation and selection. At the same time, PACEvolve~\citep{yan2026pacevolve} strengthened long-horizon search through hierarchical context management, momentum-based backtracking, island-style collaboration, and a persistent idea pool. These systems improve how the agent stores, revisits, and coordinates search trajectories over time.
The second direction improves the \emph{policy acting within the search loop}. ThetaEvolve~\citep{wang2025thetaevolve} trains the mutation policy while treating the evolving program database as the environment, showing that this dynamic search state is essential: reinforcement learning from a static starting point performs worse than learning within the non-stationary evolutionary process. TTT-Discover similarly couples policy learning with evolutionary search with an entropic objective~\citep{yuksekgonul2026learning}. These results suggest that reinforcement learning in self-evolving systems should be understood as learning over \emph{search dynamics}, rather than optimizing isolated prompts.

\subsection{Reinforcement Learning in Evolutionary Search Agents}

Two representative self-evolving systems integrate reinforcement learning into an evolutionary search agent. ThetaEvolve~\citep{wang2025thetaevolve} uses a GRPO-style objective to train the mutation policy from grouped candidates sampled from the same search state~\citep{shao2024deepseekmath}. Given rewards $\{R_1, \ldots, R_n\}$, the normalized advantage for sample $i$ is $\hat{A}_i^{\text{GRPO}} =
    \frac{R_i - \bar{R}}{\sigma_R + \epsilon_{\mathrm{num}}},
    \ 
    \bar{R} = \frac{1}{n}\sum_{j=1}^n R_j, \ 
    \sigma_R =
    \sqrt{
        \frac{1}{n}
        \sum_{j=1}^n
        (R_j - \bar{R})^2
    }
$.
TTT-Discover~\citep{yuksekgonul2026learning} instead adopts an entropic reinforcement learning objective with a KL penalty that concentrates gradient mass on exceptional rollouts~\citep{jiang2025risk}. Given rewards $\{R_1, \ldots, R_n\}$, the adaptive inverse temperature $\beta$ is selected such that
$\mathrm{KL}(q_\beta \,\|\, \text{uniform}) = \gamma$,
where \(q_\beta(i) = \exp(\beta R_i) / \sum_{j=1}^n \exp(\beta R_j)\), and the leave-one-out advantage for sample $i$ is computed as $\hat{A}_i^{\text{entropic}} =
    \frac{\exp\!\left(\beta (R_i - R_{\max})\right)}
    {Z_{-i} + \epsilon_{\mathrm{num}}}
    - 1$,
where $Z_{-i} =
    \frac{1}{n-1}
    \sum_{j \neq i}
    \exp\!\left(\beta (R_j - R_{\max})\right)$.
In TTT-Discover, the entropic objective is paired with state reuse, making it well-suited to discovery settings where a single breakthrough branch matters more than average batch quality.
These methods show that evolutionary trajectories can provide useful test-time supervision for policy learning. In many research and engineering tasks, strong mutations require domain-specific reasoning about architectural design, optimization, and system trade-offs~\citep{zhang2024wukong, zhai2024actions, guo2017deepfm}. At the same time, evaluators are often too expensive for only small rollout groups to be feasible~\citep{gao2022kuairec}. Under this regime, the choice of reinforcement learning signal inside the search loop becomes a central design decision. In addition, both train the policy as an end-to-end actor, implicitly assuming that the same model can both identify promising search directions and implement them reliably.

\section{Method}

\subsection{Agent Workflow}
\label{sec:workflow}

Figure~\ref{fig:pacerl} summarizes the full workflow. The method assumes a population-based evolutionary search agent that exposes the current parent program, recent search history, evaluator scores, and a synchronization point at rollout boundaries. At each iteration, the advisor conditions on the parent program and search history to generate and select a hypothesis. A frontier implementation model converts this hypothesis into a concrete code edit, which the task-specific scorer then evaluates. The resulting outcomes are incorporated into the evolutionary population before the corresponding policy update is performed. After optimization on that rollout batch, the updated advisor parameters are synchronized to the rollout workers and used for the next iteration.

The workflow is organized around two design choices. \S~\ref{sec:advisor} describes the advisor decomposition, which learns the strategic reasoning policy while delegating code realization to a stronger implementation model. \S~\ref{sec:hybrid} describes the search-dynamics-aware objective, which changes the source of credit assignment as the search moves from exploration to frontier refinement. This design retains the advantages of strong context and search-state management while enabling test-time policy refinement through learned, task-specific search priors. Its decoupled structure also naturally admits off-policy training, requiring only changes to the synchronization barriers imposed by the top-level orchestrator.

\subsection{Advisor Model Training}
\label{sec:advisor}

The workflow above separates implementation from reasoning. In MLE tasks (\S~\ref{sec:tasks}), high-level search reasoning and low-level code implementation have different capacity requirements. Training an open-weight model end-to-end to produce full function-level mutations often fails because the model cannot reliably implement complex candidates, causing the reward to reflect implementation success as much as idea quality (Appendix~\ref{appen:tasks}).

We therefore apply reinforcement learning to an advisor model~\citep{asawa2025train} tasked with proposing new candidate ideas. The advisor learns the strategic parts of evolutionary search, including idea generation, novelty classification, and hypothesis selection, while a stronger frontier model translates the selected hypothesis into concrete code modifications~\citep{comanici2025gemini}. This separates \emph{what to try} from \emph{how to implement it}, aligning with the broader post-training practice of developing reasoning and coding as distinct capabilities before composing them in agentic systems.

The trained policy therefore serves as an adaptive reasoning layer over the evolving search landscape. Useful mutations depend not only on the current code state, but also on the current phase of the search: whether the frontier requires broader exploration, architectural consolidation, or fine-grained refinement under a fixed evaluation budget. This division enables the advisor to internalize not only static domain knowledge but also dynamic search priors: which mutation families tend to unlock new regions of the search space early, which ideas are worth revisiting after partial progress, and which refinements are likely to yield improvements over the current frontier.

\subsection{Search Dynamics Aware Policy Optimization}
\label{sec:hybrid}

Training the advisor model requires an RL objective that remains stable when candidate evaluation is costly, and the search frontier is non-stationary. The key design issue is not only the reward scale but also the geometry of credit assignment. Early in the search, candidates often differ in mechanism and quality, so centered-score differences provide useful, dense feedback. Late in search, candidates are often near-neighbor variants of an already strong parent, so the decisive event is whether a response changes the best-of-\(k\) frontier. Our objective is designed around this transition.

This transition is especially important in realistic optimization tasks. A single candidate may require GPU training, large-scale simulation, system benchmarking, or multi-dataset validation, leading to evaluation budgets measured in minutes and hours rather than seconds. Under this regime, rollout groups are necessarily smaller, and the RL objective must extract useful learning signals from far fewer candidates.

Prior self-evolving systems~\citep{shinkaevolve, shojaee2024llm} were primarily developed for settings with inexpensive evaluators, such as mathematical verification or kernel microbenchmarks, where each candidate can be scored in seconds, and hundreds of rollouts can be generated per optimization step. Recent work on efficient evolutionary search reduces the full search horizon to a few hundred iterations~\citep{shinkaevolve, assumpccao2025codeevolve, liu2026evox}, but existing RL methods for evolutionary agents still rely on much larger rollout batches, often generating 512 candidates per training step~\citep{yuksekgonul2026learning, wang2025thetaevolve}. Under this setup, each reinforcement learning step can cost more than an entire sample-efficient evolutionary run~\citep{cemri2026adaevolve, yan2026pacevolve}. We therefore investigate how to enable robust test-time reinforcement learning within evolutionary search while retaining its sample efficiency.

\paragraph{Search phases in long-horizon evolution.}

Long-horizon evolutionary search often exhibits log-like marginal reward increase as search progresses due to the increasing difficulty of discovering new state-of-the-art solutions~\citep{alphaevolve, yuksekgonul2026learning}. Early in training, the frontier is broad and diverse: sampled candidates differ substantially in their mechanisms, implementation strategies, and quality~\citep{liu2025fitness}. In this exploratory regime, dense token-level relative feedback is particularly valuable when candidate solutions differ significantly. 

Later, the search enters a refinement regime, where new state-of-the-art solutions become more difficult to discover. Candidates become local variants of already strong solutions, reward gains exhibit diminishing returns, and absolute score gaps compress toward the level of evaluator noise~\citep{cheng2025barbarians}. The optimization question changes from "which mutation class is broadly better?" to "which candidate meaningfully changes the frontier?" In this regime, entropic weighting over-concentrates on reward outliers, while GRPO amplifies small numerical differences into disproportionately large gradient magnitudes, often causing optimization instability. Recent lines of work have systematically analyzed GRPO's deficiency in small-batch, low-reward-variance regimes, such as high variance~\citep{zhou2026demystifying, han2026ebpo} and bias for high-likelihood solutions~\citep{plyusov2026f}. 

Figure~\ref{fig:main_results} illustrates these failure modes in practice. The auxiliary traces reveal unstable optimization behavior: entropy can collapse as the objective over-commits to exploitation, while gradient norms spike when compressed rewards are amplified into large updates. These dynamics motivate a training objective whose credit geometry changes with the search phase.

\begin{figure*}[t]
    \centering
    \begin{subfigure}[t]{0.32\textwidth}
        \centering
        \includegraphics[width=\linewidth]{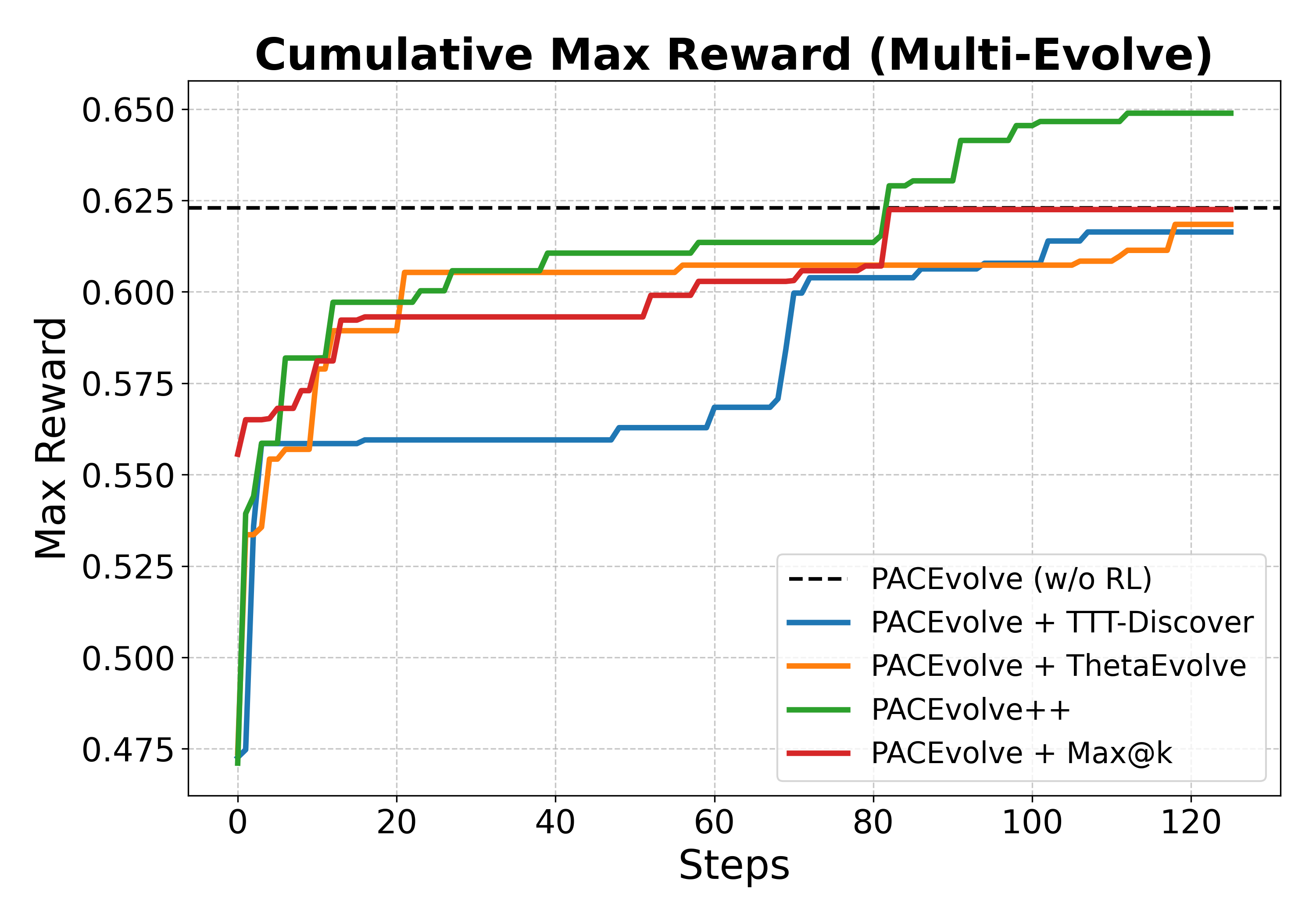}
        \caption{Multi-Evolve cumulative max reward.}
    \end{subfigure}
    \hfill
    \begin{subfigure}[t]{0.32\textwidth}
        \centering
        \includegraphics[width=\linewidth]{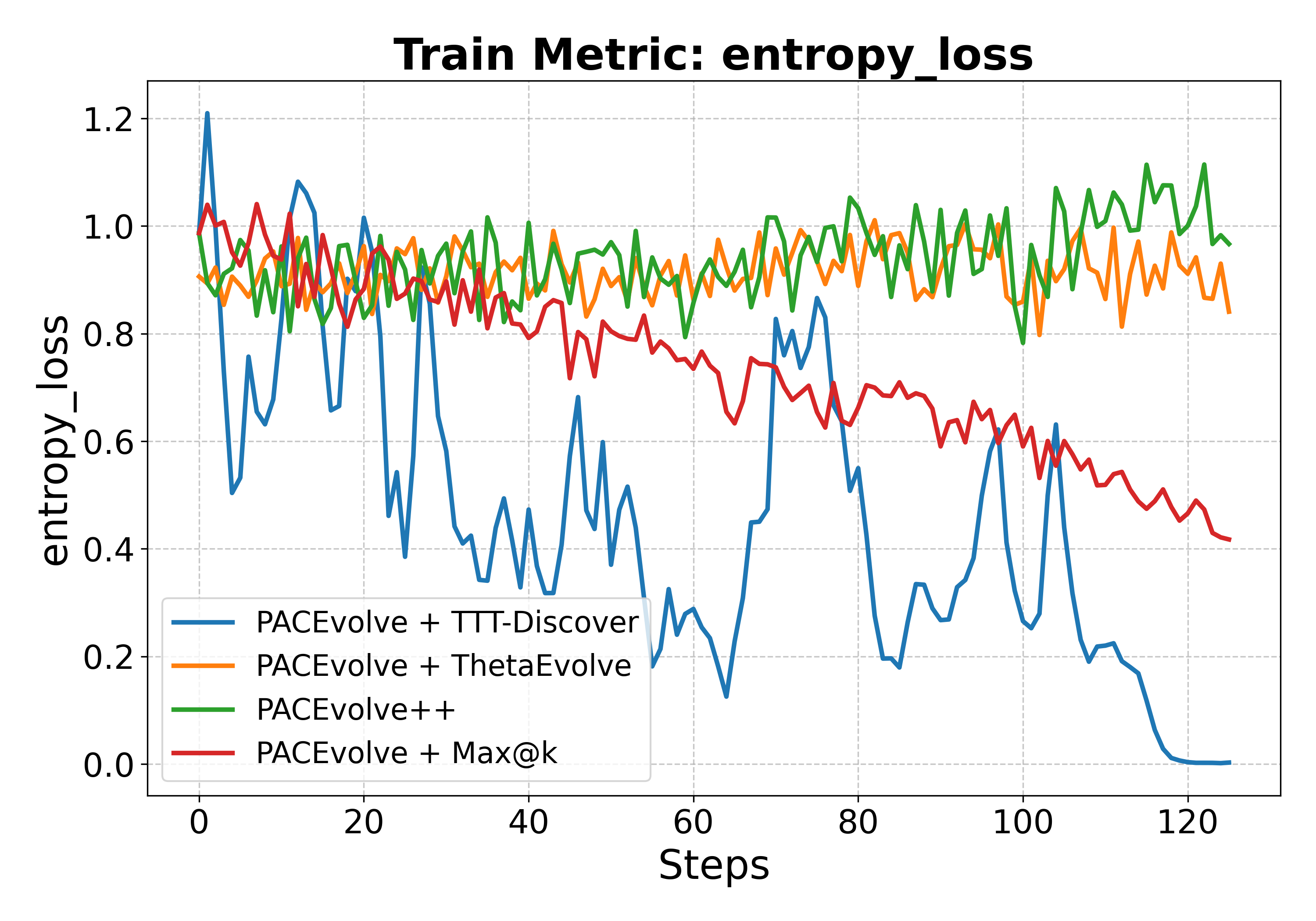}
        \caption{Multi-Evolve policy entropy.}
    \end{subfigure}
    \hfill
    \begin{subfigure}[t]{0.32\textwidth}
        \centering
        \includegraphics[width=\linewidth]{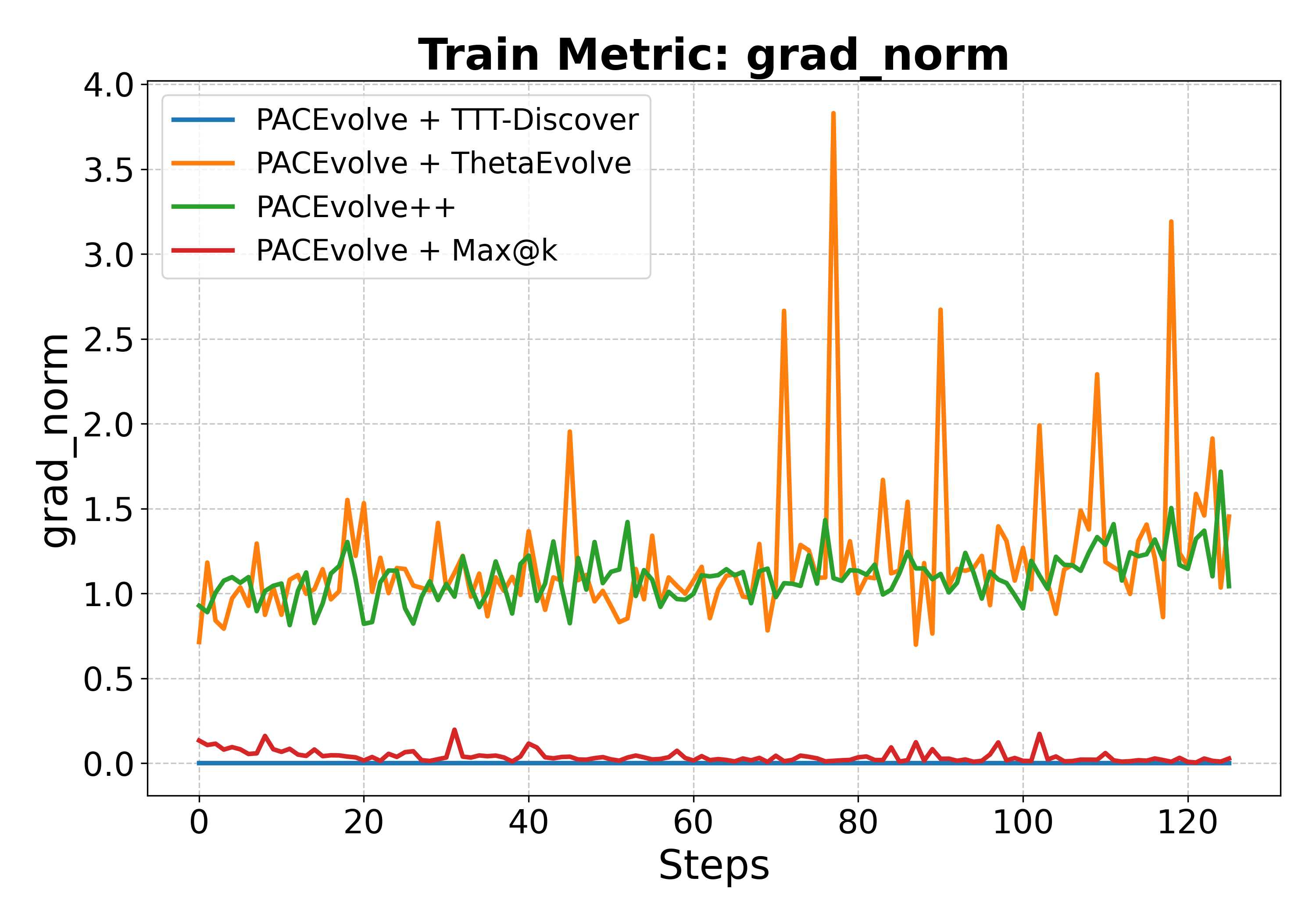}
        \caption{Multi-Evolve gradient norm.}
    \end{subfigure}
    \caption{Training dynamics for DeepSeek-R1-0528-Qwen3-8B on Multi-Evolve. \method{} reaches the best final reward while avoiding the instability patterns observed in baselines.}
    \label{fig:main_results}
    \vspace{-10pt}
\end{figure*}

\paragraph{Phase-aligned advantage design.}

To mitigate the above challenges, we design the training signal around the search dynamics themselves. In the exploratory regime, a raw group-relative baseline preserves dense within-group credit assignment without the late-stage variance blow-up:
$\hat{A}_i^{G} = R_i - \bar{R}$~\citep{liu2025understanding}.
To encourage exploration, we also adopt the asymmetric clipping introduced in DAPO, so that more rare but promising tokens can still receive meaningful positive updates~\citep{yu2025dapo}.

In the refinement regime, we use a pass@$k$-based marginal-contribution signal (PKPO)~\citep{walder2025pass}. Given $N$ sampled responses from search state \(x\) with rewards $g_1, \ldots, g_N$, PKPO constructs unbiased gradient weights $w_i$ such that
\begin{equation}
    \nabla_\theta
    \mathbb{E}\!\left[\max(g_1, \ldots, g_k)\right]
    =
    \mathbb{E}\!\left[
        \sum_{i=1}^N
        w_i
        \nabla_\theta
        \log \pi_\theta(a_i \mid x)
    \right].
\end{equation}
The corresponding PKPO weight can be written as a normalized sum of best-of-$k$ scores over all size-$k$ subsets that contain sample $i$:
\begin{equation}
    w_i =
    \frac{1}{\binom{N}{k}}
    \sum_{\substack{I \subseteq \{1, \ldots, N\} \\ |I| = k,\ i \in I}}
    \max_{j \in I} g_j.
\end{equation}
Equivalently, this is $\frac{k}{N}$ times the conditional average over size-$k$ subsets that contain $i$.
In practice, we use the low-variance SLOO$_{k-1}$ estimator, which turns this into an explicit marginal-contribution signal by subtracting the best alternative available when $i$ is removed:
\begin{equation}
    \hat{A}_i^{\mathrm{top-k}} = w_i^{\mathrm{SLOO}} =
    \frac{1}{\binom{N}{k}}
    \sum_{\substack{I \subseteq \{1, \ldots, N\} \\ |I| = k,\ i \in I}}
    \left(
        \max_{j \in I} g_j
        -
        \max_{b \in I \setminus \{i\}} g_b
    \right).
\end{equation}

\begin{theorem}[Scale-conditioned credit assignment under reward compression]
\label{thm:reward-collapse}
Let \(g_i^{(\delta)}=c+\delta r_i\) be a reward batch with fixed ranking, compression scale \(\delta>0\), base mean \(\bar r\), and base standard deviation \(\sigma_r\). Let \(\Phi_{\epsilon_{\mathrm{num}}}(B)_i=(B_i-\mu(B))/(\sigma(B)+\epsilon_{\mathrm{num}})\). For the raw group-relative branch \(A_i^G(\delta)=g_i^{(\delta)}-\bar g^{(\delta)}\) and the SLOO branch \(w_i^{\mathrm{SLOO}}(\delta)\), the standardized branches satisfy
\[
\Phi_{\epsilon_{\mathrm{num}}}(A^G(\delta))_i
=
\frac{\delta(r_i-\bar r)}{\delta\sigma_r+\epsilon_{\mathrm{num}}},
\qquad
\Phi_{\epsilon_{\mathrm{num}}}(w^{\mathrm{SLOO}}(\delta))_i
=
\frac{\delta\left(w_i^{\mathrm{SLOO}}(1)-\mu(w^{\mathrm{SLOO}}(1))\right)}
{\delta\sigma(w^{\mathrm{SLOO}}(1))+\epsilon_{\mathrm{num}}}.
\]
Both standardized branch vectors have squared \(L_2\) norm at most \(N\).
\end{theorem}

Theorem~\ref{thm:reward-collapse} formalizes the scale-conditioned view used by our objective (proof in Appendix~\ref{appen:training_stability}). When the corresponding branch standard deviation dominates \(\epsilon_{\mathrm{num}}\), standardization removes the global affine reward scale and preserves the branch-specific credit ordering. In early search, reward variance is large enough that standardized group-relative feedback is a well-conditioned centered score-difference signal. As the search progresses and candidates become increasingly similar, the more important distinction is the geometry of credit assignment: SLOO\(_{k-1}\) assigns credit according to whether a response changes a best-of-\(k\) frontier. This frontier-contribution geometry is invariant to affine reward rescaling, aligning with late-stage refinement, where absolute gaps are small, but the identity of frontier-changing candidates remains informative.

\paragraph{Phase-adaptive advantage computation.}

In our training setup, each rollout iteration contains a group of candidates sampled from their respective evolutionary search process. Let $\mathcal{G}_t$ denote this rollout group at iteration $t$. The raw group-relative and SLOO signals can have different numerical ranges, and PPO-style clipped objectives are sensitive to arbitrary advantage scale. We therefore standardize each scalar estimator within the current rollout group before mixing: $\tilde{A}_i^{(\cdot)} =
    \frac{\hat{A}_i^{(\cdot)} - \mu_{\mathcal{G}_t}(\hat{A}^{(\cdot)})}
    {\sigma_{\mathcal{G}_t}(\hat{A}^{(\cdot)}) + \epsilon_{\mathrm{num}}},
    \  i \in \mathcal{G}_t$.
This step makes the two branches numerically comparable before clipping. The standardized group-relative branch remains a dense z-score over rollout rewards, while the standardized PKPO branch is an affine transform of a frontier-contribution score. Thus, the phase-adaptive mixture changes the source and semantics of credit assignment rather than merely changing the update scale. If the corresponding standard deviation is non-finite or below the numerical threshold $\epsilon_{\mathrm{skip}}$, suggesting that the branch has collapsed to numerical noise, we skip this gradient update rather than normalizing an uninformative signal.
We then form a mixed scalar score
\begin{equation}
\label{eq:hybrid}
    A_i^{\text{mix}}(t)
    =
    (1-\alpha_t)\tilde{A}_i^{G}
    +
    \alpha_t \tilde{A}_i^{\text{top-k}},
\end{equation}

The phase schedule, therefore, changes which signal dominates rather than inadvertently changing the overall update magnitude. We scale $\alpha_t$ linearly from 0 to 1 throughout the course of training.

\paragraph{Progress-normalized reward shaping.}

The RL objective uses a task-specific score as the raw objective, as in any evolutionary search framework~\citep{alphaevolve, openevolve, shinkaevolve}. Let \(y\) denote a successfully parsed finite score for a given task. Each task specifies an optimization direction together with lower and upper normalization bounds \(y_{\min}\) and \(y_{\max}\). When explicit score-transform bounds are not provided, these are set from the task configuration as
$y_{\min} = \min(y_{\mathrm{init}}, y_{\mathrm{target}})$,
$y_{\max} = \max(y_{\mathrm{init}}, y_{\mathrm{target}})$.
We then compute a direction-aware normalized progress variable
$u(y) = \mathrm{clamp}\!\left(\dfrac{y - y_{\min}}{y_{\max} - y_{\min}}, 0, 1 \right)$ if the metric is maximized, and change the numerator to $y_{\max} - y$ otherwise. We then define the RL reward as $R_{\mathrm{RL}}(y) = c\, u(y)^{\alpha_r}$,
where \(c>0\) is a positive multiplier and \(\alpha_r>0\) is a shaping exponent. In practice, we reduce this to a linearly scaled progress reward on \([0,5]\). Scores outside the configured range are clamped before normalization.
If evaluation fails or the result cannot be parsed, we assign \(-1.0\) as the reward. Parsed but non-finite scores are also mapped to \(-1.0\). Equivalently,
\begin{equation}
    R =
    \begin{cases}
        R_{\mathrm{RL}}(y), & \text{if } y \text{ is successfully parsed and finite}, \\
        -1.0, & \text{otherwise}.
    \end{cases}
\end{equation}

This transformation places heterogeneous task metrics, including both maximization and minimization objectives, into a shared progress-based reward scale for RL training.

\paragraph{Loss function.}

Our estimator produces a scalar advantage per sampled response. Concretely, Eq.~\ref{eq:hybrid} defines a response-level score $A_i^{\mathrm{mix}}(t)$ for response $i$, which is then broadcast to all response tokens: \(A_{i,\tau}^{\mathrm{tok}} = A_i^{\mathrm{mix}}(t)\) for all \((i,\tau) \in \mathcal{T}_t\).
We then optimize a masked clipped surrogate objective over valid response tokens~\citep{schulman2017proximal}:
\begin{equation}
    \mathcal{L}(\theta)
    =
    -\mathbb{E}_{(i,\tau) \sim \mathcal{T}_t}
    \left[
        \min\!\left\{
        r_{i,\tau}(\theta)A_{i,\tau}^{\mathrm{tok}},
        \operatorname{clip}\!\left(r_{i,\tau}(\theta),1-\epsilon_{\mathrm{lo}},1+\epsilon_{\mathrm{hi}}\right)A_{i,\tau}^{\mathrm{tok}}
        \right\}
    \right].
\end{equation}

Here $\mathcal{T}_t = \{(i,\tau) : i \in \mathcal{G}_t,\; m_{i,\tau}=1\}$ denotes the valid response tokens in rollout group $\mathcal{G}_t$, $m_{i,\tau}$ is the response-token loss mask, and $r_{i,\tau}(\theta)$ is the token-level importance ratio.

\section{Experiments} \label{sec:exp}

\subsection{Task Selection} \label{sec:tasks}
We evaluate \method's performance on a variety of real-world machine-learning-related engineering and research tasks, spanning algorithm design for model routing~\citep{liu2024deepseek}, improvements over the state-of-the-art recommender models~\citep{gao2022kuairec, ye2026fuxi}, and model design for protein engineering~\citep{tran2026rapid}. These tasks are grounded in real-world challenges and require innovative solutions for further improvements. Shared training settings and task-specific evaluator timeouts are reported in Appendix~\ref{appen:config}.

\subsubsection{Expert-parallelism Load Balancing}

\paragraph{Problem.} Mixture-of-experts (MoE) models route computation through specialized expert subnetworks, but balancing load across devices during parallel inference remains challenging~\citep{shazeer2017outrageously}. The EPLB task asks for an algorithm that, given a workload profile of per-expert demand, assigns experts to devices to minimize the maximum per-device load while remaining computationally efficient.

\paragraph{Evaluation.} Candidates are tested on expert-load profiles derived from production MoE traces. We report two metrics: \emph{balancedness}, which measures the uniformity of device load, and \emph{speed}, defined as the inverse of the algorithm's wall-clock time. The final score is their arithmetic mean, as in \citet{cheng2025barbarians}.

\paragraph{Evolution surface.} The evolvable block implements only the assignment logic: its input is an expert-load tensor and its output is a device-assignment map. The evaluation harness, data loading, and metrics are fixed.

\subsubsection{Sequential Recommendation}

\paragraph{Problem.} Sequential recommendation aims to predict a user's next interaction from their history. Our KuaiRec task uses a FuXi-linear-style sequential recommender~\citep{ye2026fuxi, gao2022kuairec}. Concretely, the benchmark evolves a fixed-budget next-item ranking model on KuaiRec, a fully observed user-item interaction dataset from the Kuaishou short-video platform, with long user histories and time-aware sequence modeling.

\paragraph{Evaluation.} Each candidate model is trained for 16 epochs with sampled softmax and evaluated by full-catalog ranking. We report NDCG@10, Hit Rate@10, and MRR, and optimize their arithmetic mean. Each candidate is subject to a 1{,}200-second wall-clock budget; exceeding this budget results in evaluation failure.

\paragraph{Evolution surface.} The evolvable block covers sequence feature construction and the FuXi-linear-style encoder/scoring logic. In particular, candidates can modify how raw histories are converted into item, timestamp, and positional features, as well as the multi-channel sequence mixer, pooling strategy, and item-scoring module. The data pipeline, training loop, sampled-softmax objective, and evaluation protocol are fixed.

\subsubsection{Protein Fitness Extrapolation}

\paragraph{Problem.} Predicting the fitness effect of multiple simultaneous protein mutations from single- and double-mutant training data is challenging~\citep{romero2009exploring}. The Multi-Evolve benchmark measures extrapolation: models are trained on wild-type, single, and double mutants, and must then predict fitness for mutants of order three or higher~\citep{tran2026rapid}.

\paragraph{Evaluation.} For each dataset, we report Pearson correlation ($r$) and Precision@5, defined as the fraction of top-5 predictions that are truly top-5. We define the combined score as $0.7 \times \overline{r} + 0.3 \times \overline{\text{P@5}}$, averaged across datasets.

\paragraph{Evolution surface.} The evolvable block covers mutation featurization, pairwise epistatic interactions, regularization and calibration, sample weighting across mutation orders, and lightweight ensembling.

\subsection{Baselines}

We evaluate all RL variants in the same long-horizon evolutionary search harness. We compare against methods that integrate RL into evolutionary search (ThetaEvolve~\citep{wang2025thetaevolve}, TTT-Discover~\citep{yuksekgonul2026learning}, and Max@$k$ training~\citep{walder2025pass}) by varying the training setups during advisor training. We also compare against a no-RL scaffold baseline to isolate the effect of test-time advisor training. This setup preserves a strong adaptive workflow while directly comparing approaches for efficient test-time training during evolution, providing a fair testbed for various reinforcement learning recipes. We evaluate on Qwen3.5-4B and DeepSeek-R1-0528-Qwen3-8B to demonstrate our method across different model sizes using Gemini-3.1-pro-preview for candidate implementation. More details on baselines and experiment setups are discussed in Appendix~\ref{appen:config}.

\subsection{Results}

Figures~\ref{fig:main_results_8b} and~\ref{fig:main_results_4b} compare the main training trajectories. Across EPLB, Sequential Recommendation, and Protein Fitness Prediction, \method{} consistently provides the strongest final reward while optimizing smoothly and converging the fastest. We note that on EPLB, both \method{} and the non-RL PACEvolve baseline reach a saturated near-optimal solution~\citep{cheng2025barbarians}, but \method{} uses only half of the evolution budget. On Sequential Recommendation and Protein Fitness Extrapolation, our method converges to a better solution than the baselines. The entropy trace further indicates that \method{} exhibits fewer spikes and instabilities than baseline methods. In Appendix~\ref{appen:disaggre}, we provide disaggregated metrics for each metric we aim to optimize jointly.

\begin{figure*}[t]
    \centering
    \begin{subfigure}[t]{0.32\textwidth}
        \centering
        \includegraphics[width=\linewidth]{figures/plot_max_score_evolution_8b_Multi-Evolve.png}
        \caption{Multi-Evolve max reward.}
    \end{subfigure}
    \hfill
    \begin{subfigure}[t]{0.32\textwidth}
        \centering
        \includegraphics[width=\linewidth]{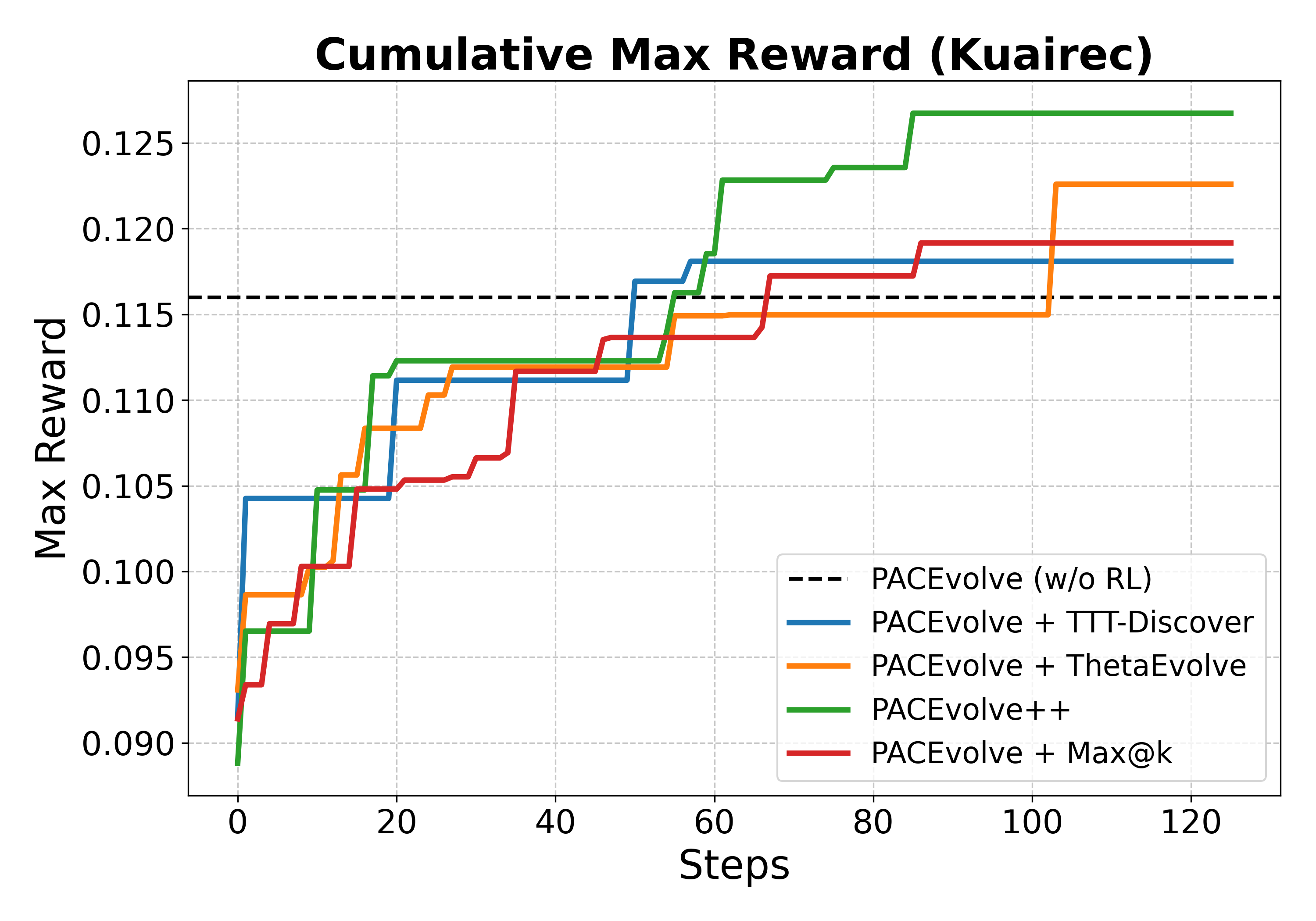}
        \caption{KuaiRec max reward.}
    \end{subfigure}
    \hfill
    \begin{subfigure}[t]{0.32\textwidth}
        \centering
        \includegraphics[width=\linewidth]{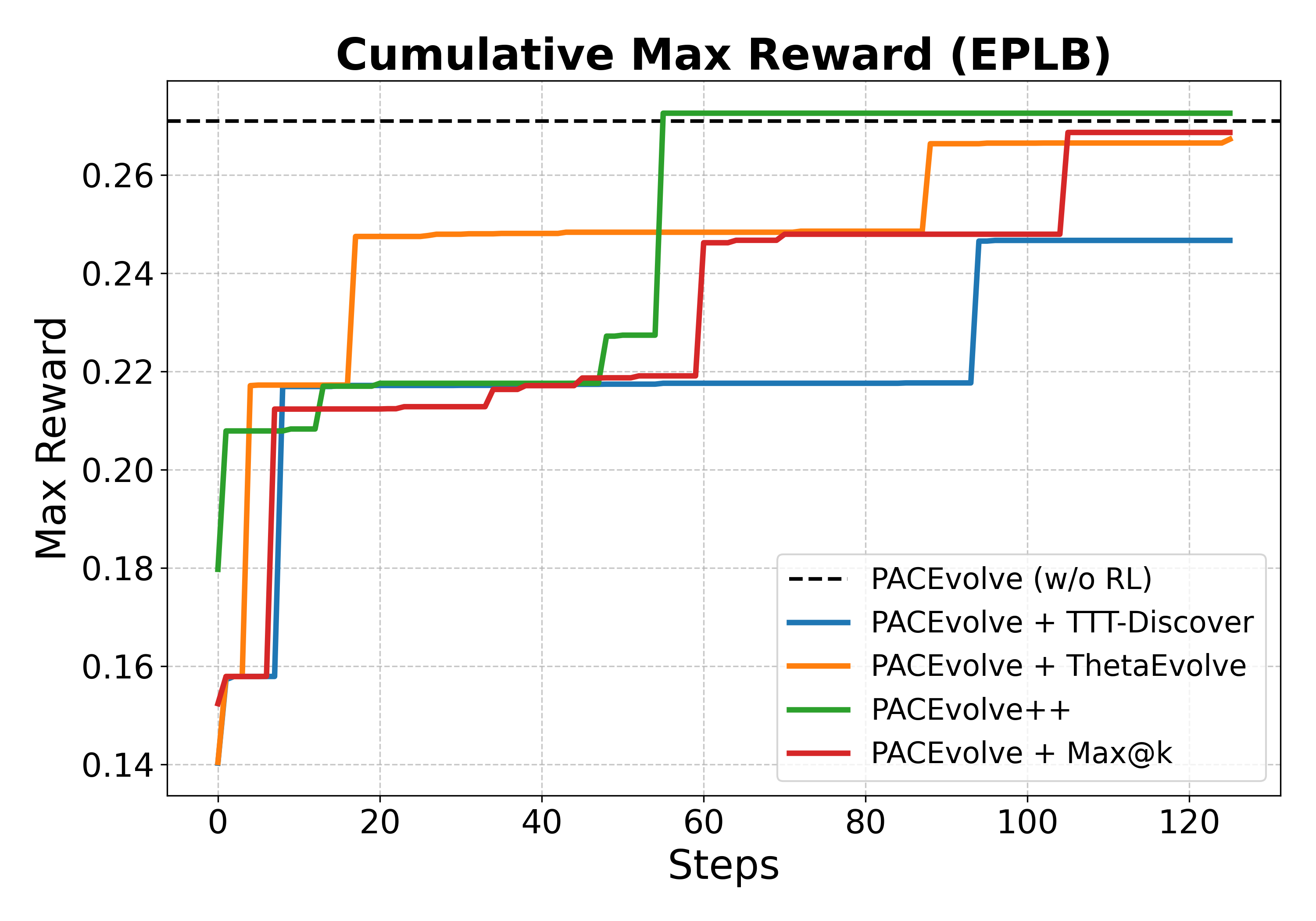}
        \caption{EPLB max reward.}
    \end{subfigure}
    \caption{Comparison of different RL algorithms on DeepSeek-R1-0528-Qwen3-8B across three tasks. \method{} reaches the best final reward and converges the fastest.}
    \label{fig:main_results_8b}
    \vspace{-10pt}
\end{figure*}

\begin{figure*}[t]
    \centering
    \begin{subfigure}[t]{0.32\textwidth}
        \centering
        \includegraphics[width=\linewidth]{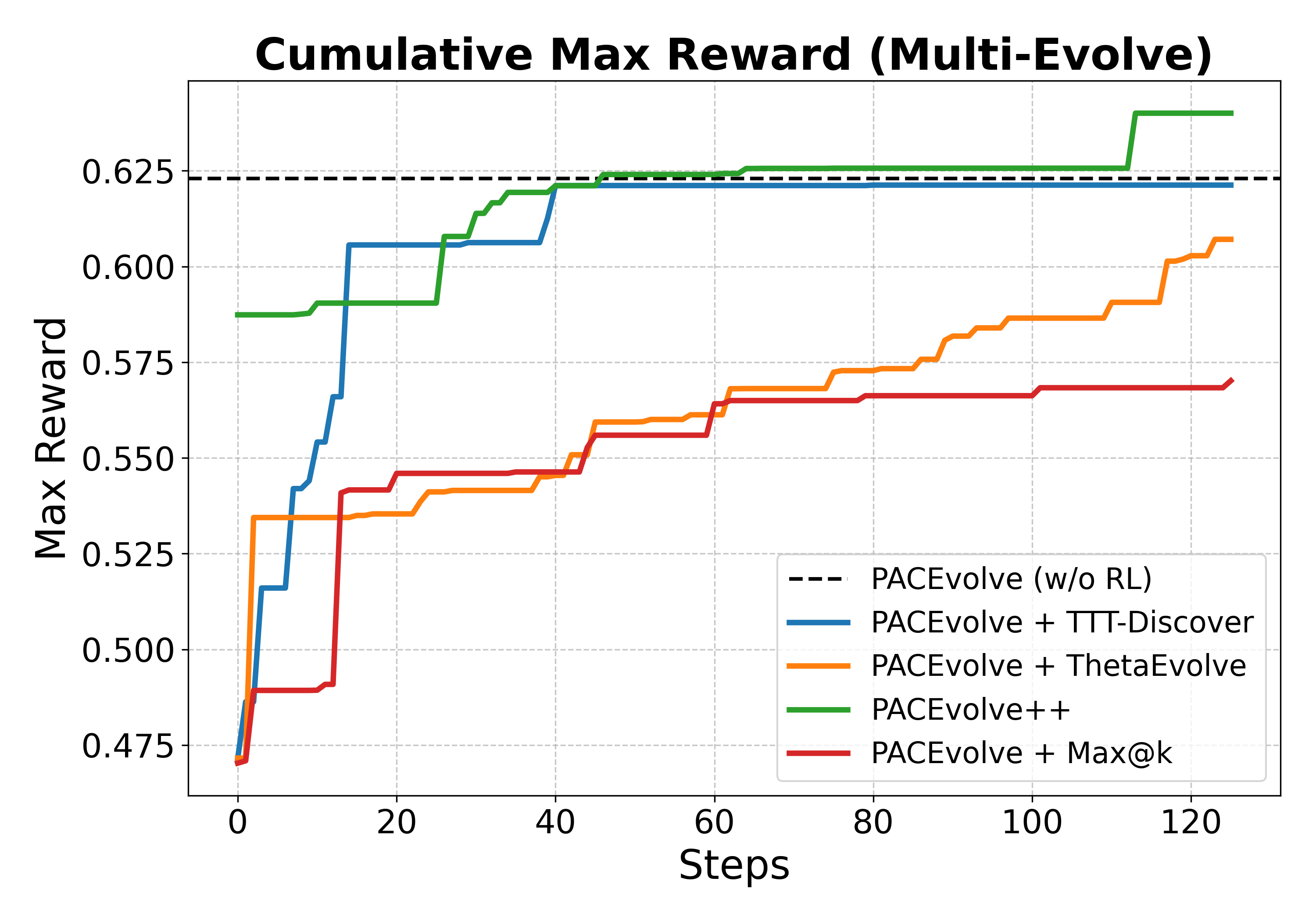}
        \caption{Multi-Evolve max reward.}
    \end{subfigure}
    \hfill
    \begin{subfigure}[t]{0.32\textwidth}
        \centering
        \includegraphics[width=\linewidth]{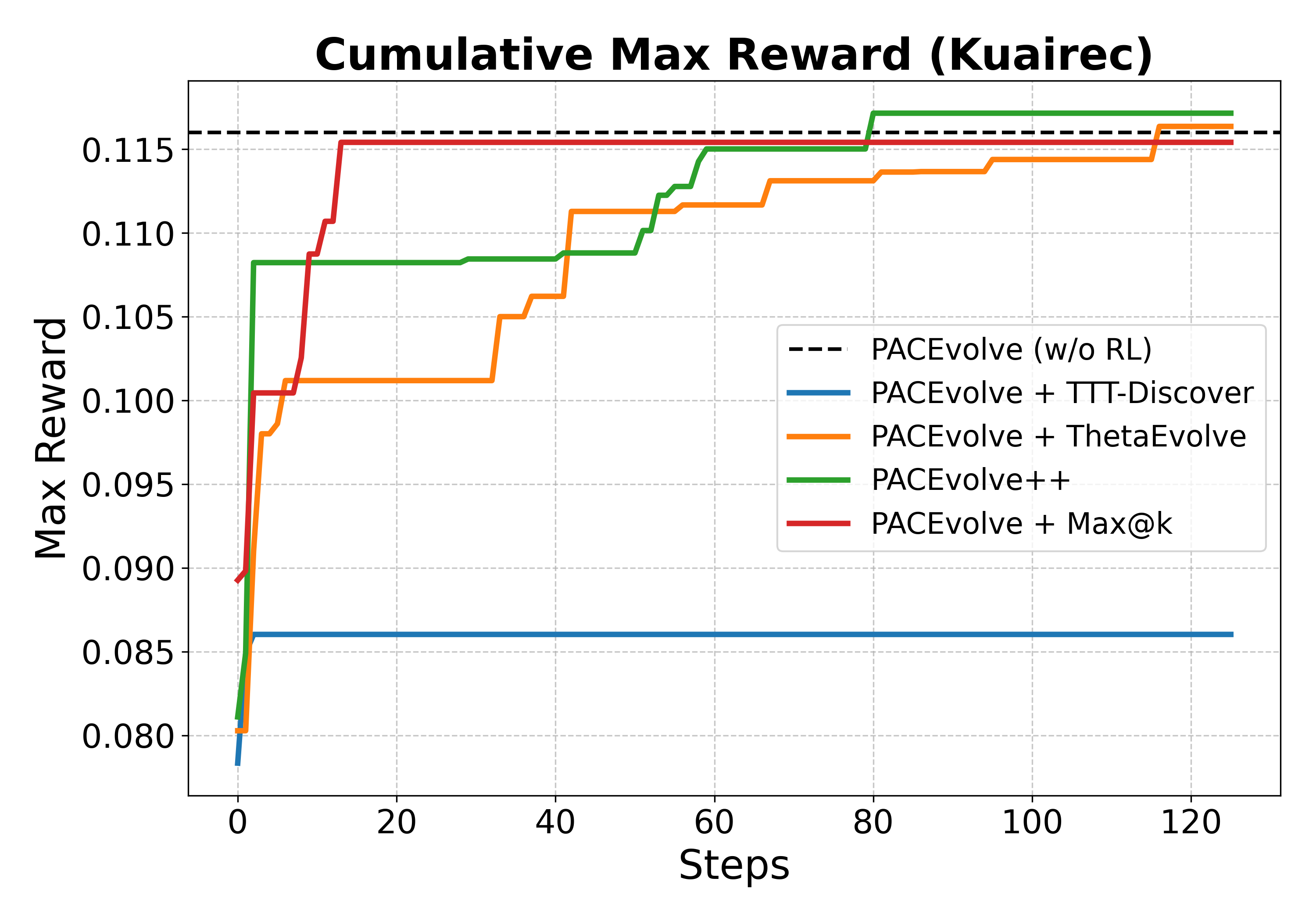}
        \caption{KuaiRec max reward.}
    \end{subfigure}
    \hfill
    \begin{subfigure}[t]{0.32\textwidth}
        \centering
        \includegraphics[width=\linewidth]{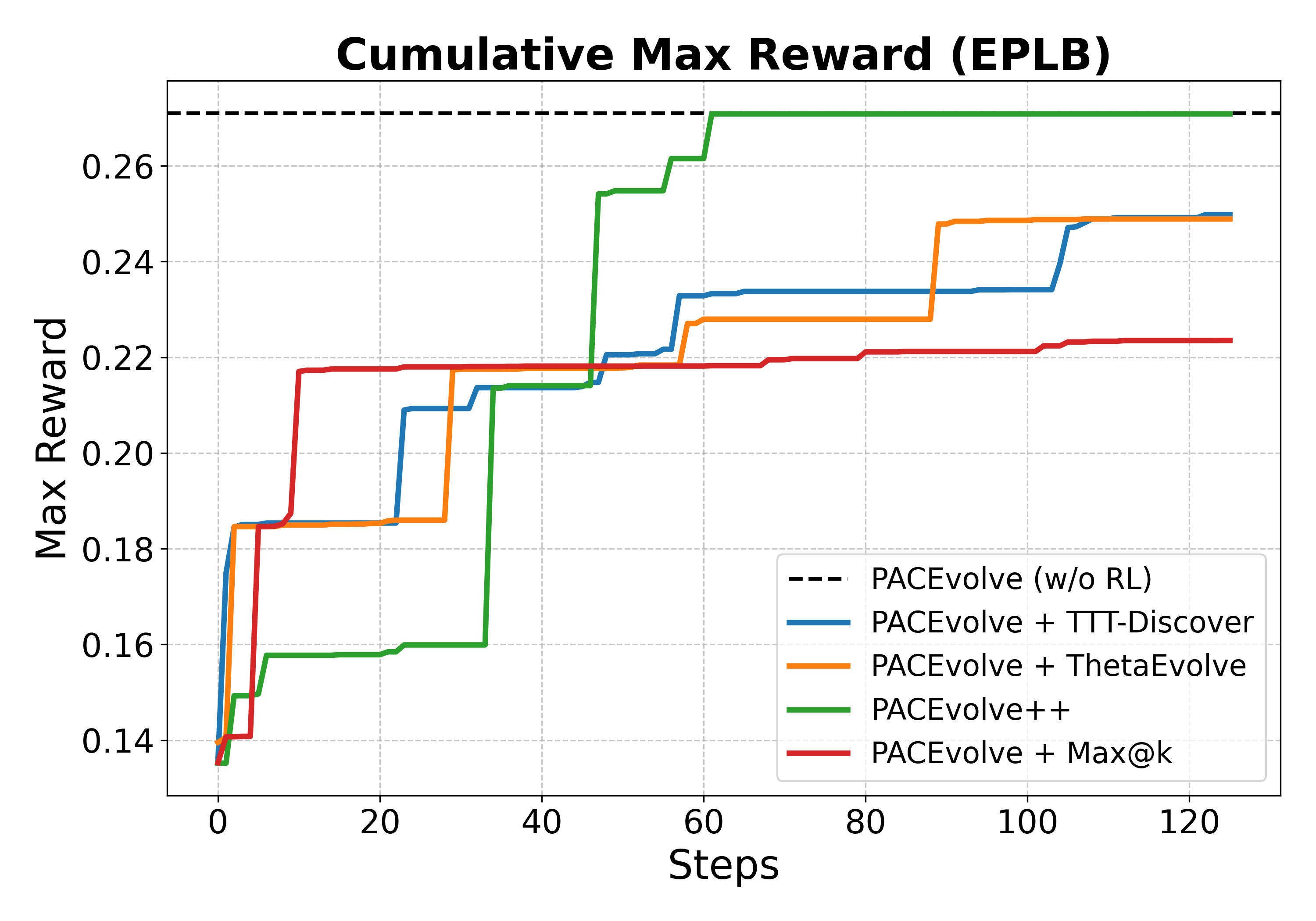}
        \caption{EPLB max reward.}
    \end{subfigure}
     \caption{Comparison of different RL algorithms on Qwen3.5-4B across three tasks. \method{} demonstrates faster convergence to better results.}
    \label{fig:main_results_4b}
    \vspace{-10pt}
\end{figure*}

\paragraph{Analysis.}
The auxiliary metrics clarify why the baselines stall. ThetaEvolve's GRPO-style objective remains competitive in raw reward for much of EPLB. Still, Appendix~\ref{appen:diag} shows repeated late-stage gradient-norm spikes, including several excursions above $2$ and a peak above $4$, consistent with variance blow-up once grouped rewards compress. PKPO, shown as Max@$k$ in the figures, exhibits the opposite pathology: its entropy decreases almost monotonically from roughly $1.0$ to below $0.4$, indicating that it commits to exploitation too early and loses diversity before the search frontier is saturated. The entropic objective is the least stable overall due to concentrated reward distribution. 

Our search-dynamic-aware objective avoids these pathologies by matching the training signal to the search phase. Early in training, the grouped-relative branch maintains high entropy to sustain exploration, unlike Max@$k$, which collapses exploration too quickly. Later, the frontier-contribution branch improves refinement without inheriting GRPO's gradient spikes. This behavior is visible in the appendix diagnostics: \method{} keeps gradient norms in a comparatively narrow band around $1$ while maintaining materially higher entropy than Max@$k$, and these smoother dynamics translate into the strongest final search performance.

\section{Related work}

\paragraph{Evolutionary search agents.}
Evolutionary search with language models has developed along two closely related threads. FunSearch~\citep{funsearch} and AlphaEvolve~\citep{alphaevolve} showed that repeated propose-evaluate-select loops can turn LLMs into effective algorithmic search operators. At the same time, open-weight successors such as OpenEvolve~\citep{openevolve} broadened the set of accessible domains. PACEvolve~\citep{yan2026pacevolve} shifted attention from single-step mutation quality to long-horizon search organization, emphasizing context compression, backtracking, and collaborative exploration. Our work is complementary: we retain a strong scaffold but focus on developing a stronger reasoning policy within it.

\paragraph{Test-time training}
Test-time training aims to modify the model during inference for better performance~\citep{sun2020test, sun2024learning}. Prior work has explored test-time training in a variety of setups~\citep{zuo2025ttrl, tandon2025end, gandelsman2022test}. More recently, methods have been developed to integrate test-time training into evolutionary search agents, as the evolutionary search process can naturally generate on-policy data for reinforcement learning. ThetaEvolve~\citep{wang2025thetaevolve} trains the mutation policy against the evolving program database, highlighting the importance of learning within the non-stationary search process rather than relying solely on static prompts. TTT-Discover~\citep{yuksekgonul2026learning} combines an entropic objective with search-time state reuse and PUCT-style traversal~\citep{alphaevolve}, emphasizing discovery settings in which rare breakthroughs matter more than average batch quality.

\paragraph{Policy optimization for LLMs.}
PPO~\citep{schulman2017proximal} and its variants underpin RLHF and related post-training methods. GRPO~\citep{shao2024deepseekmath} replaces the learned value function with grouped baselines; DAPO~\citep{yu2025dapo} adds asymmetric clipping to encourage exploration of low probability tokens. Dr.~GRPO~\citep{liu2025understanding} removes both standard deviation and length normalization biases from GRPO. Pass@k training~\citep{chen2025pass} develops an entropy-guided approach to optimize pass@k for verifiable tasks. PKPO~\citep{walder2025pass} also targets the pass@$k$ objective and derives unbiased, low-variance gradient estimators via combinatorial weighting. 

\paragraph{Advisor models and small-model steering.}
Advisor-model approaches train compact open-weight models to generate instance-specific guidance for stronger frozen models~\citep{asawa2025train}. Our formulation adopts the same high-level separation of concerns: the trained advisor is responsible for strategic reasoning, whereas the larger frontier model is responsible for faithful implementation. In the self-evolving setting, this design allows task-specific search priors to be learned in the smaller model while preserving the coding strength of the larger implementation model.

\section{Conclusion}
We introduced \method{}, an advisor-style reinforcement learning framework for self-evolving agents that learns task-specific search priors under expensive evaluation regimes. By decoupling high-level reasoning from implementation and aligning the optimization objective with search dynamics, our approach stabilizes training in practical machine learning research and engineering settings where existing methods struggle. Empirically, \method{} achieves stronger and more stable search performance across diverse machine learning engineering tasks. These results highlight the importance of improving the reasoning policy, rather than just the search scaffold, to scale self-evolving agents to realistic domains.

\bibliography{pacerl}

\newpage
\appendix

\section{Limitations}
Due to the high costs of both RL training and evolutionary search and limited resources, exacerbated by the fact that evaluating each evolutionary candidate involves training a model, we could not repeat the experiments or run them over a longer horizon. We leave it to future work to further scale up our experiments or train models with stronger coding capabilities that may handle evolution end-to-end.

\section{Training configuration}
\label{appen:config}

We show the default training configurations used in the experiments in Table~\ref{tab:shared_hparams}. We use the prompt templates (Appendix~\ref{appen:prompts}) and evolution setups from \citet{yan2026pacevolve} for candidate generation, selection, and code implementation. We set the temperature to 1 for all prompts. Our experiments are performed in an online on-policy setting; each of the n evolutionary search threads generates a candidate, which is then used for training in the same step. The experiments are performed on A2 instances on GCP. 

\subsection{Task complexity} \label{appen:complexity}

We note that the tasks we selected are more complex to implement than those in existing work~\citep{wang2025thetaevolve}. Therefore, we found that tasking small open-weight models (with 4B to 8B parameters) with end-to-end evolution results in a low success rate in implementation correctness, thereby biasing the reward toward ideas that are valid when implemented correctly. This also makes it infeasible for an end-to-end ThetaEvolve-style RL training. However, we note that this is mainly a model capacity concern when tasking a compact open-weight model with complex coding tasks. This motivated our design to separate idea generation from code implementation. General coding capability for small, open-weight models is an important concern but out of scope for our study; we leave it to future work to improve few-shot capabilities in implementing a research prototype for complex MLE tasks. 

\begin{table}[h]
\caption{Hyperparameter setups}
\label{tab:shared_hparams}
\centering
\begin{tabular}{lc}
\toprule
Setting & Value \\
\midrule
Evolution iterations & 1000 \\
Samples per prompt ($n$) & 8 \\
Top $k$ & 4 \\
Optimizer & AdamW \\
Learning rate & 1e-6 \\
Weight decay & 0.1 \\
Adam $\beta_1$, $\beta_2$ & 0.9, 0.98 \\
Gradient steps per rollout & 1 \\
Clip $\epsilon$ / $\epsilon_h$ & 0.2 / 0.28 \\
EPLB timeout & 600s \\
KuaiRec timeout & 1{,}200s \\
Multi-Evolve timeout & 1{,}200s \\
\bottomrule
\end{tabular}
\end{table}

\section{Task details} \label{appen:tasks}

\subsection{EPLB}

The EPLB task is drawn from DeepSeek-V3's infrastructure for MoE model serving~\citep{liu2024deepseek}. Given a workload tensor of per-expert activation counts across a batch, the algorithm assigns experts to parallel devices so that (i) the maximum per-device workload is minimized and (ii) the assignment procedure remains fast. The evolvable code block takes the workload tensor as input and returns a device-assignment map. Workload profiles are derived from the public expert-load dataset introduced in \citet{cheng2025barbarians}. 

\subsection{KuaiRec}

KuaiRec is a fully observed user-item interaction dataset from Kuaishou's short-video platform, containing roughly 7{,}176 users, 10{,}728 items, and 12.5 million interactions~\citep{gao2022kuairec}. The current benchmark instantiates a FuXi-linear-style sequential recommender with a maximum history length of 1{,}024, an embedding width of 128, four sequence-mixing blocks, and separate retention, temporal, and positional channels. The evolvable surface covers the sequence encoder and scoring logic: candidates can redesign item-, timestamp-, and position-aware token features, the multi-channel sequence mixer, the sequence summarization mechanism, and the item-scoring module. The surrounding scaffold remains fixed, including 16 epochs of sampled-softmax training, full-catalog evaluation, and the relaxed 1{,}200-second evaluator budget.

\subsection{Multi-Evolve}

Multi-Evolve evaluates combinatorial protein fitness prediction~\citep{tran2026rapid}. Proteins are mutated at multiple sites simultaneously, and the goal is to predict the joint fitness effect when training data contains only lower-order mutants. Under the same settings as \citet{tran2026rapid}, models are trained on wild-type, single, and double mutants, and then predict fitness for mutants with three or more substitutions. The evolvable block covers mutation featurization, pairwise epistatic interaction terms, regularization, sample weighting, and lightweight ensembling. Evaluation is performed across multiple protein datasets using Pearson correlation and Precision@5.

\section{Additional Results}

\subsection{Training Diagnostics} \label{appen:diag}
We show more training diagnostics metrics below (Figures~\ref{fig:diag_multi_8b}, ~\ref{fig:diag_multi_4b}, ~\ref{fig:diag_kuai_8b}, ~\ref{fig:diag_kuai_4b}, ~\ref{fig:diag_eplb_8b}, ~\ref{fig:diag_eplb_4b}). These figures further show that \method{} not only achieves the best performance among other RL algorithms, but also has the most stable training, exhibiting the least unexpected increases or decreases in entropy or gradient norm.

\begin{figure*}[t]
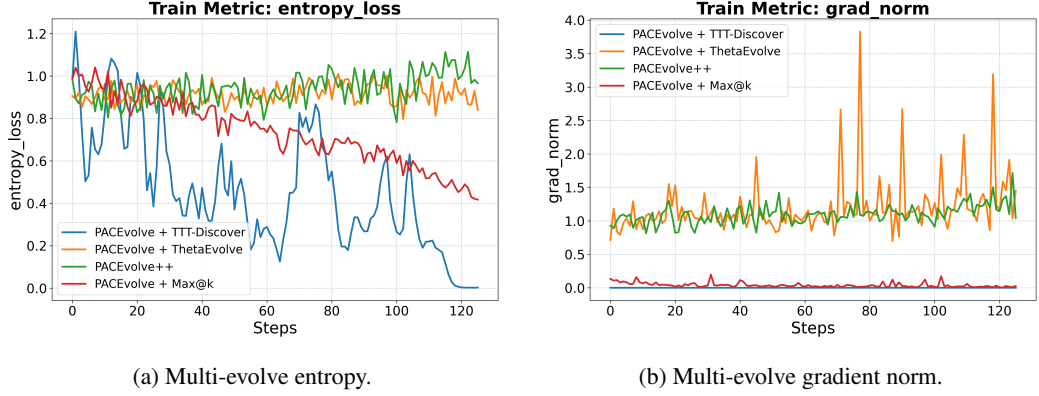

    \centering
    \begin{subfigure}[t]{0.49\textwidth}
        \centering
        \includegraphics[width=\linewidth]{figures/plot_entropy_loss_8b_Multi-Evolve.png}
        \caption{Multi-evolve entropy.}
    \end{subfigure}
    \hfill
    \begin{subfigure}[t]{0.49\textwidth}
        \centering
        \includegraphics[width=\linewidth]{figures/plot_grad_norm_8b_Multi-Evolve.png}
        \caption{Multi-evolve gradient norm.}
    \end{subfigure}
    \caption{Multi-Evolve training dynamics of 8B models. ThetaEvolve exhibits large gradient-norm spikes, Max@$k$ steadily collapses entropy, and TTT-Discover remains unstable with repeated entropy collapses. \method{} remains comparatively well conditioned on both metrics.}
    \label{fig:diag_multi_8b}
\end{figure*}

\begin{figure*}[t]
    \centering
    \begin{subfigure}[t]{0.49\textwidth}
        \centering
        \includegraphics[width=\linewidth]{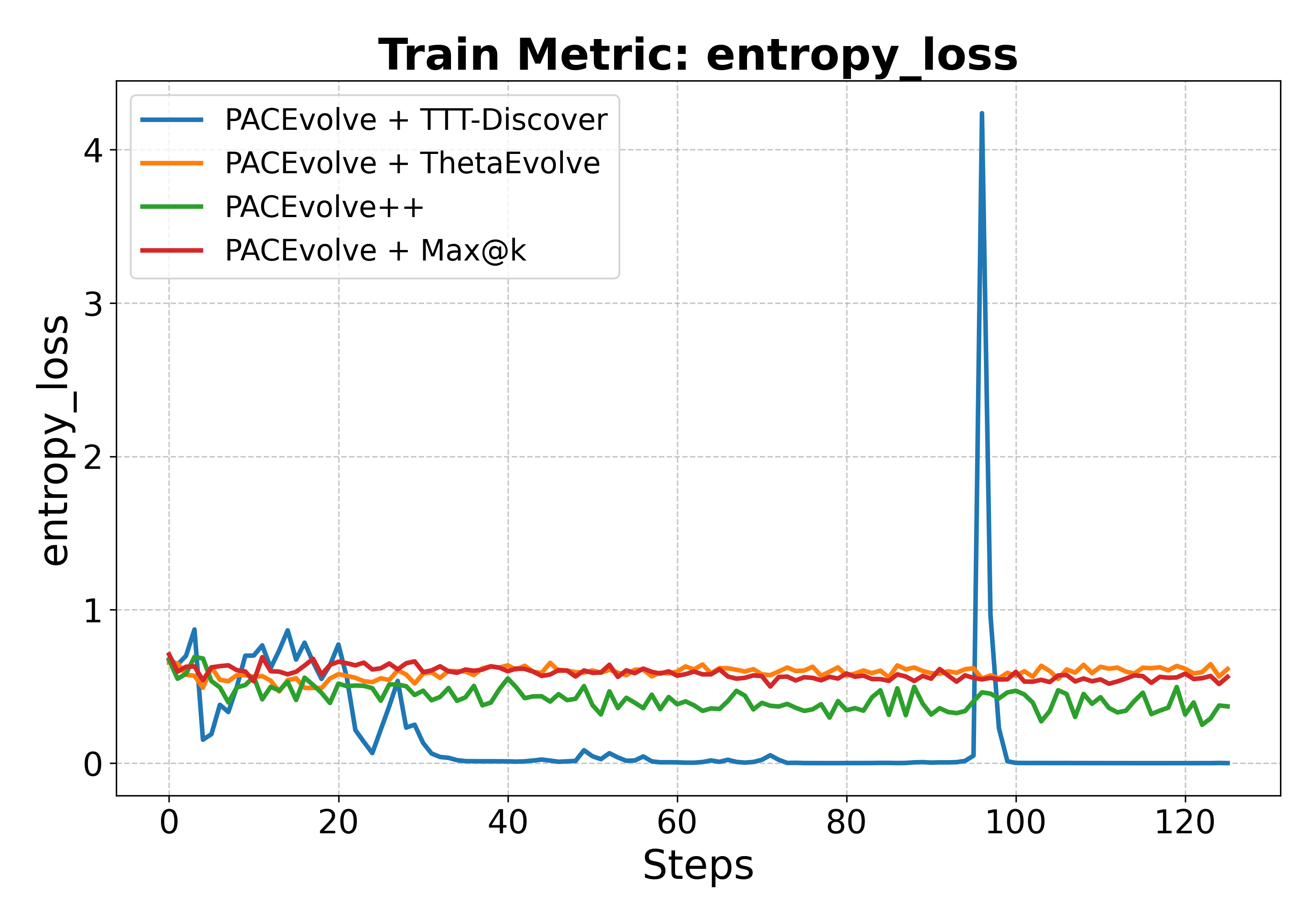}
        \caption{Multi-evolve entropy.}
    \end{subfigure}
    \hfill
    \begin{subfigure}[t]{0.49\textwidth}
        \centering
        \includegraphics[width=\linewidth]{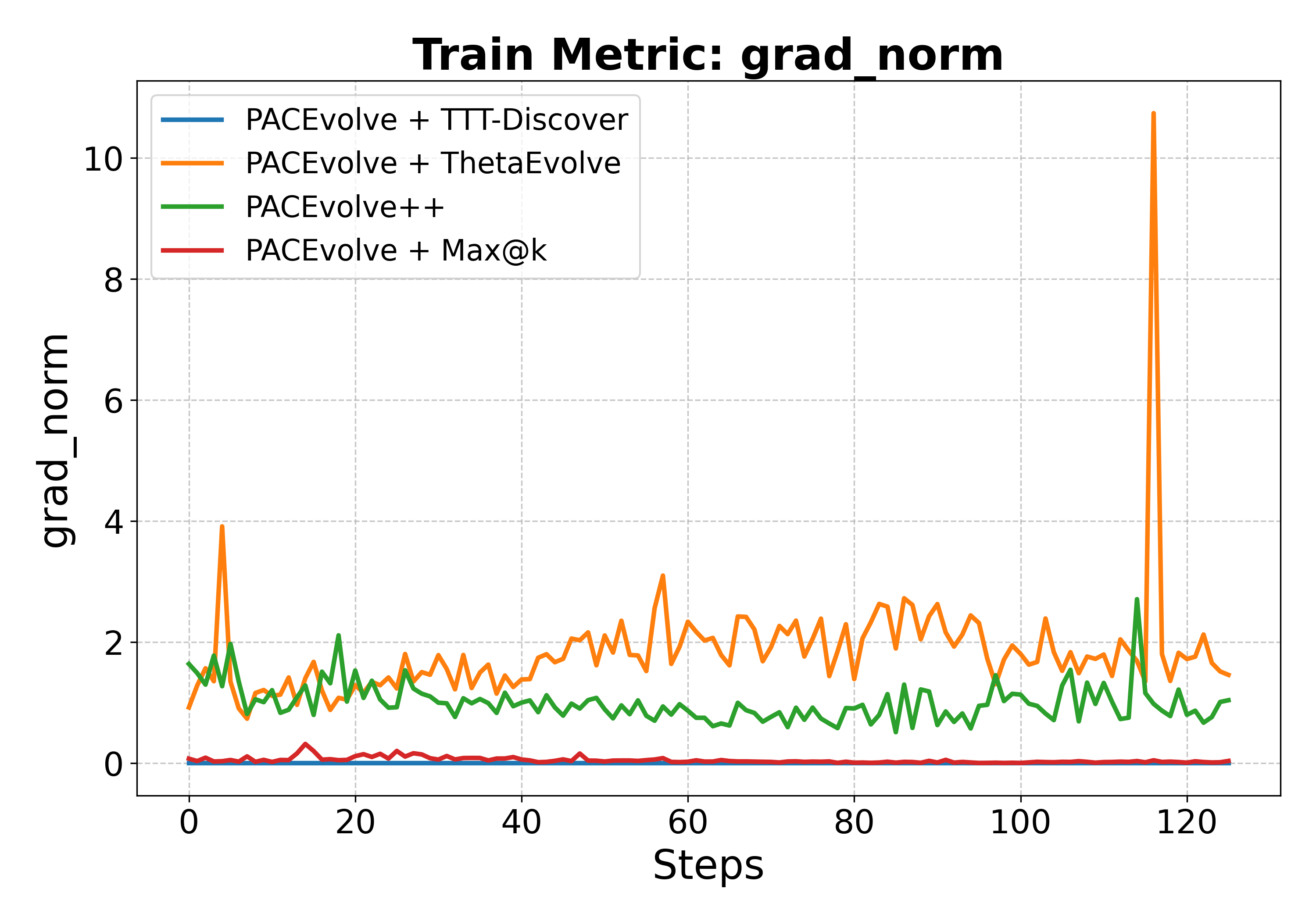}
        \caption{Multi-evolve gradient norm.}
    \end{subfigure}
    \caption{Multi-Evolve training dynamics of 4B models. \method{} remains the most stable on auxiliary metrics.}
    \label{fig:diag_multi_4b}
\end{figure*}

\begin{figure*}[t]
    \centering
    \begin{subfigure}[t]{0.49\textwidth}
        \centering
        \includegraphics[width=\linewidth]{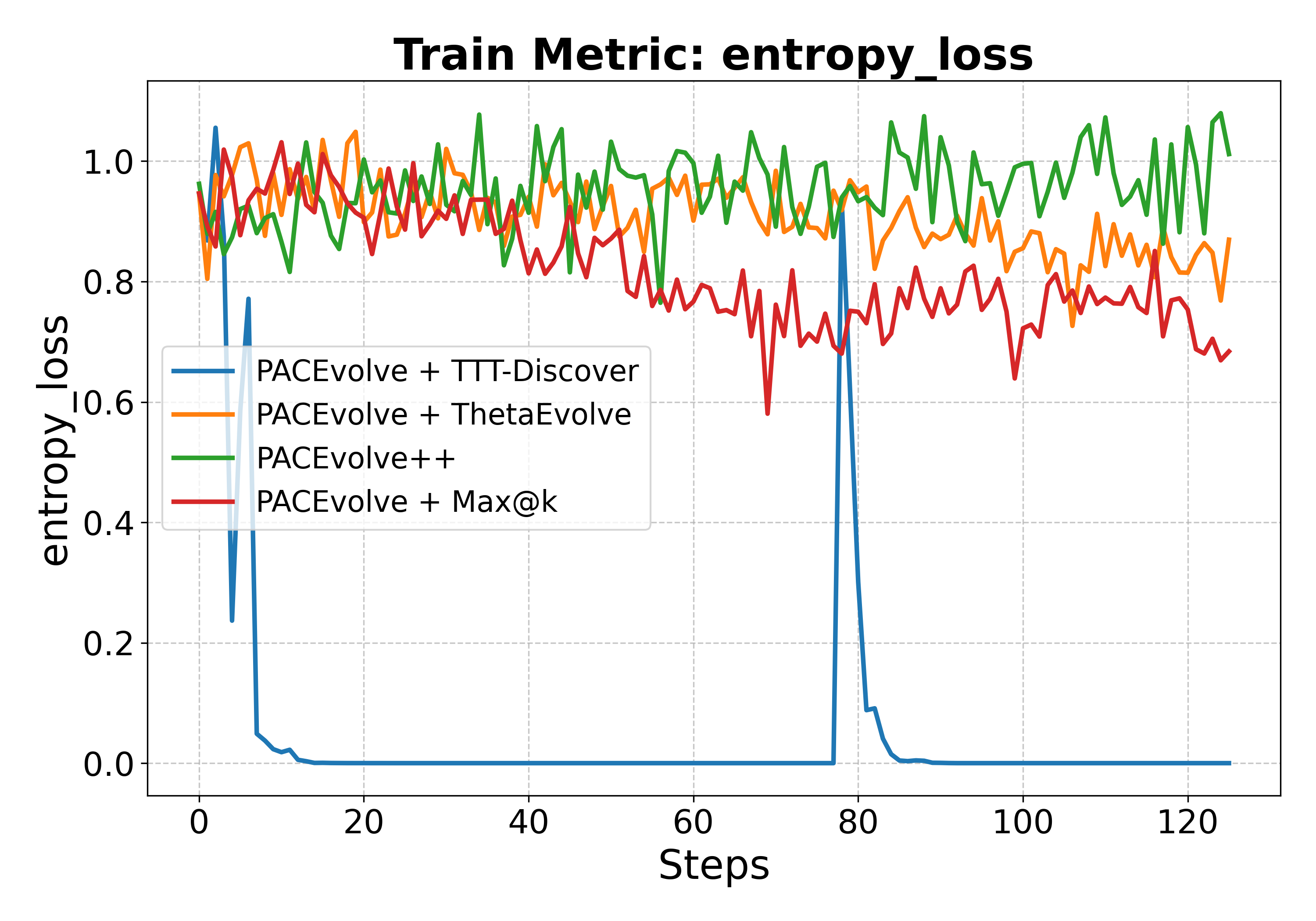}
        \caption{KuaiRec entropy.}
    \end{subfigure}
    \hfill
    \begin{subfigure}[t]{0.49\textwidth}
        \centering
        \includegraphics[width=\linewidth]{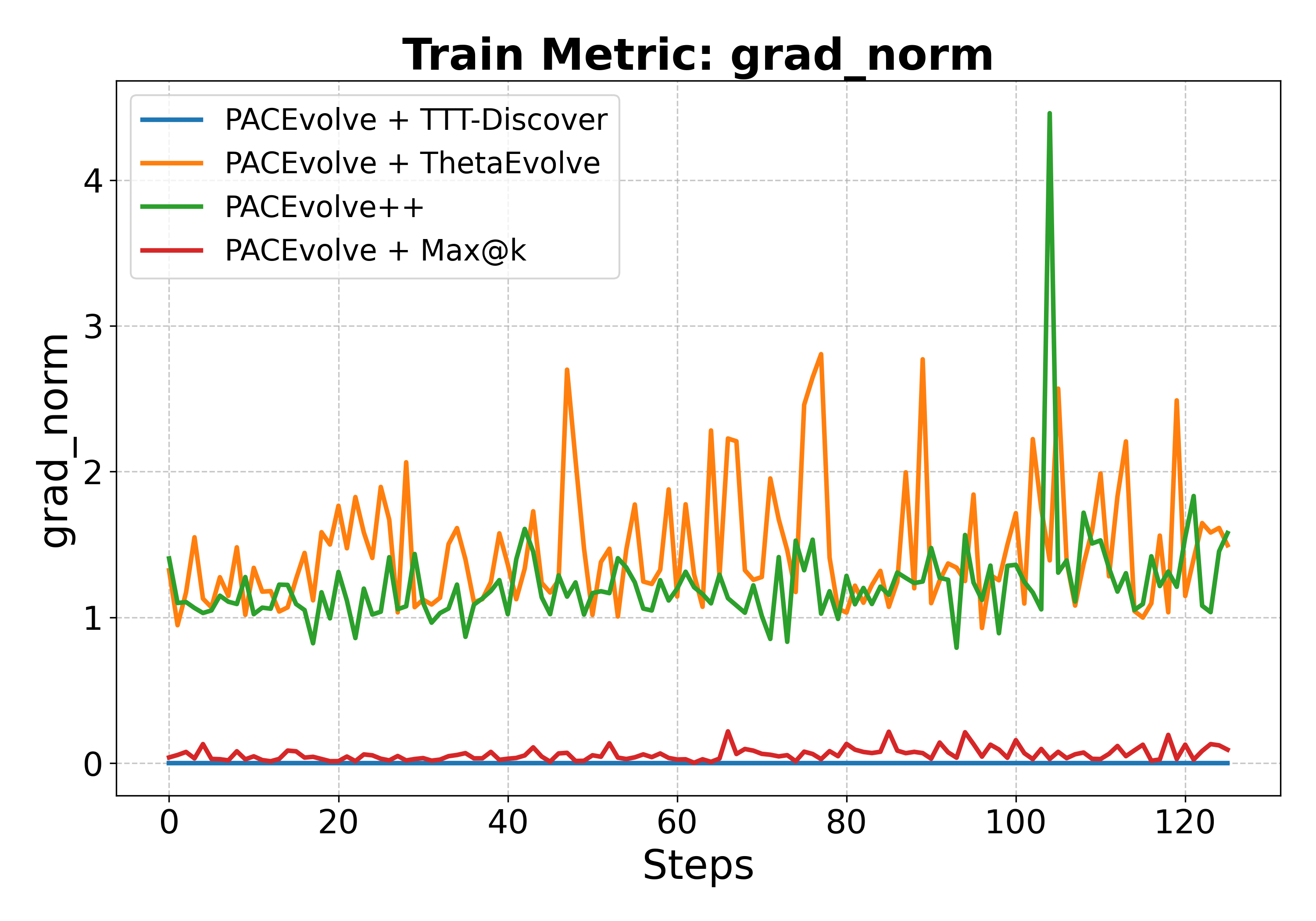}
        \caption{KuaiRec gradient norm.}
    \end{subfigure}
    \caption{KuaiRec training dynamics of 8B models. ThetaEvolve exhibits large gradient-norm spikes, Max@$k$ steadily collapses entropy, and TTT-Discover training collapsed quickly. \method{} remains comparatively well conditioned on both metrics.}
    \label{fig:diag_kuai_8b}
\end{figure*}

\begin{figure*}[t]
    \centering
    \begin{subfigure}[t]{0.49\textwidth}
        \centering
        \includegraphics[width=\linewidth]{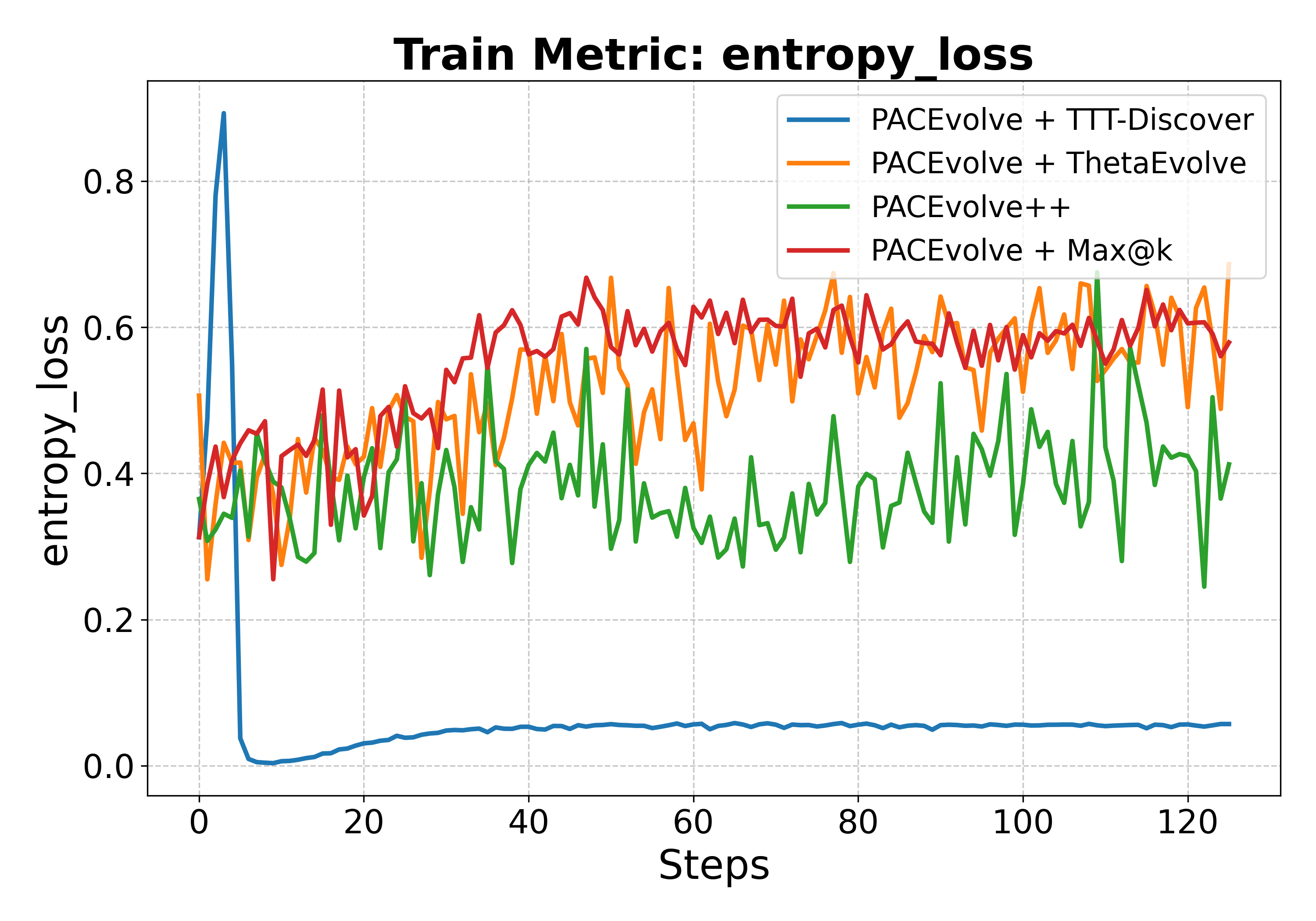}
        \caption{KuaiRec entropy.}
    \end{subfigure}
    \hfill
    \begin{subfigure}[t]{0.49\textwidth}
        \centering
        \includegraphics[width=\linewidth]{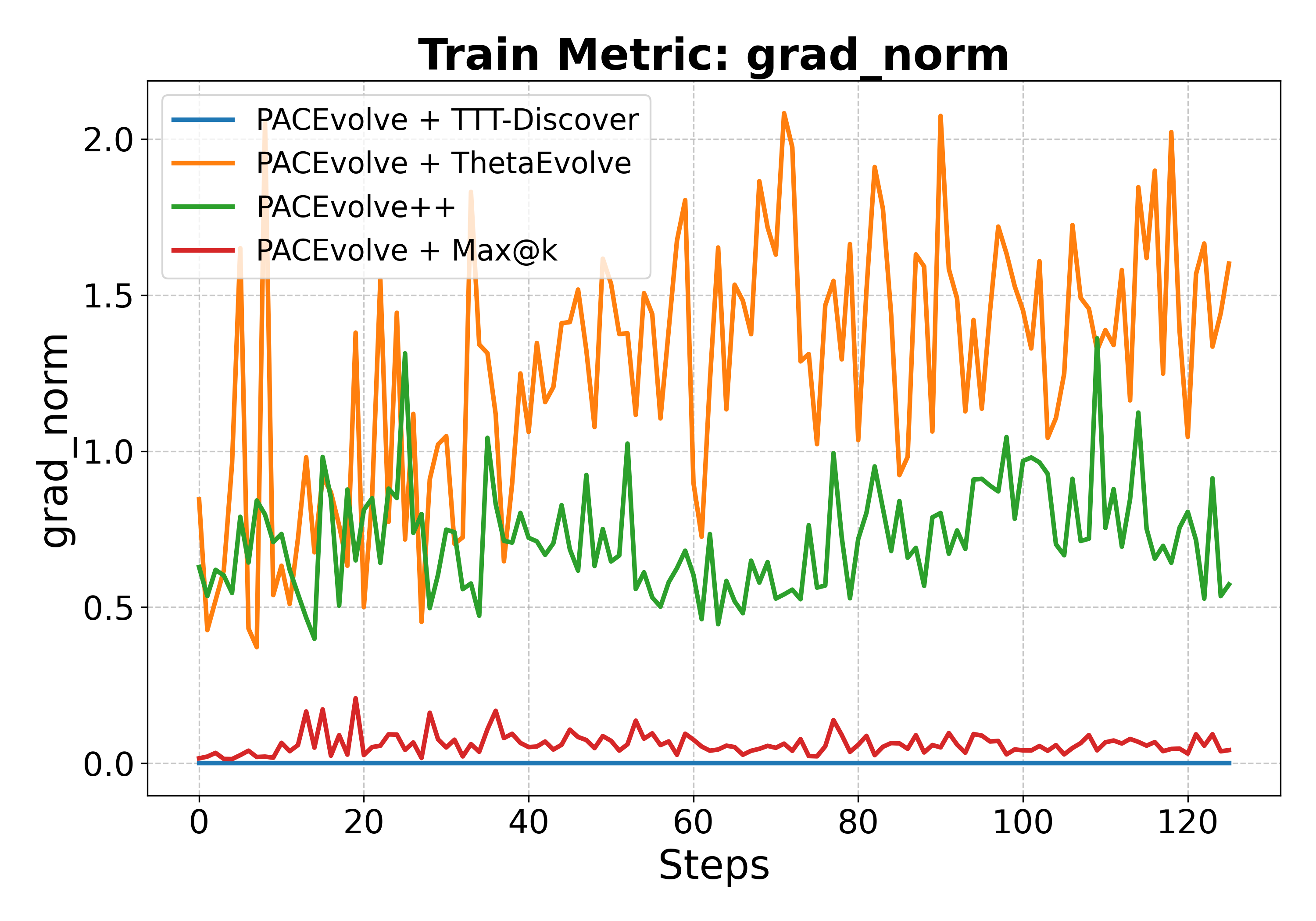}
        \caption{KuaiRec gradient norm.}
    \end{subfigure}
    \caption{KuaiRec training dynamics of 4B models. ThetaEvolve exhibits large gradient-norm spikes, Max@$k$ steadily collapses entropy, and TTT-Discover training collapsed quickly. \method{} remains comparatively well conditioned on both metrics.}
    \label{fig:diag_kuai_4b}
\end{figure*}

\begin{figure*}[t]
    \centering
    \begin{subfigure}[t]{0.49\textwidth}
        \centering
        \includegraphics[width=\linewidth]{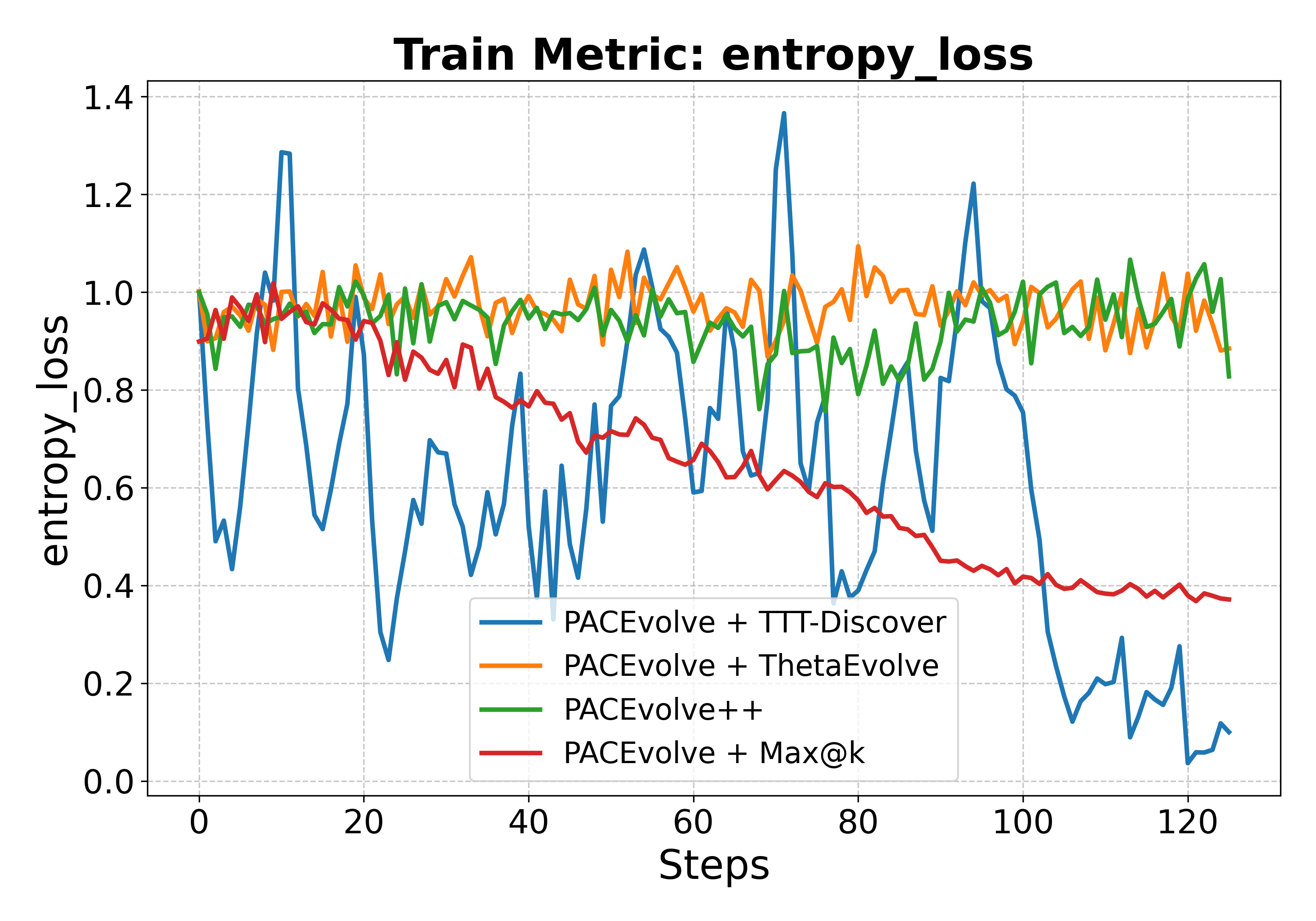}
        \caption{EPLB entropy.}
    \end{subfigure}
    \hfill
    \begin{subfigure}[t]{0.49\textwidth}
        \centering
        \includegraphics[width=\linewidth]{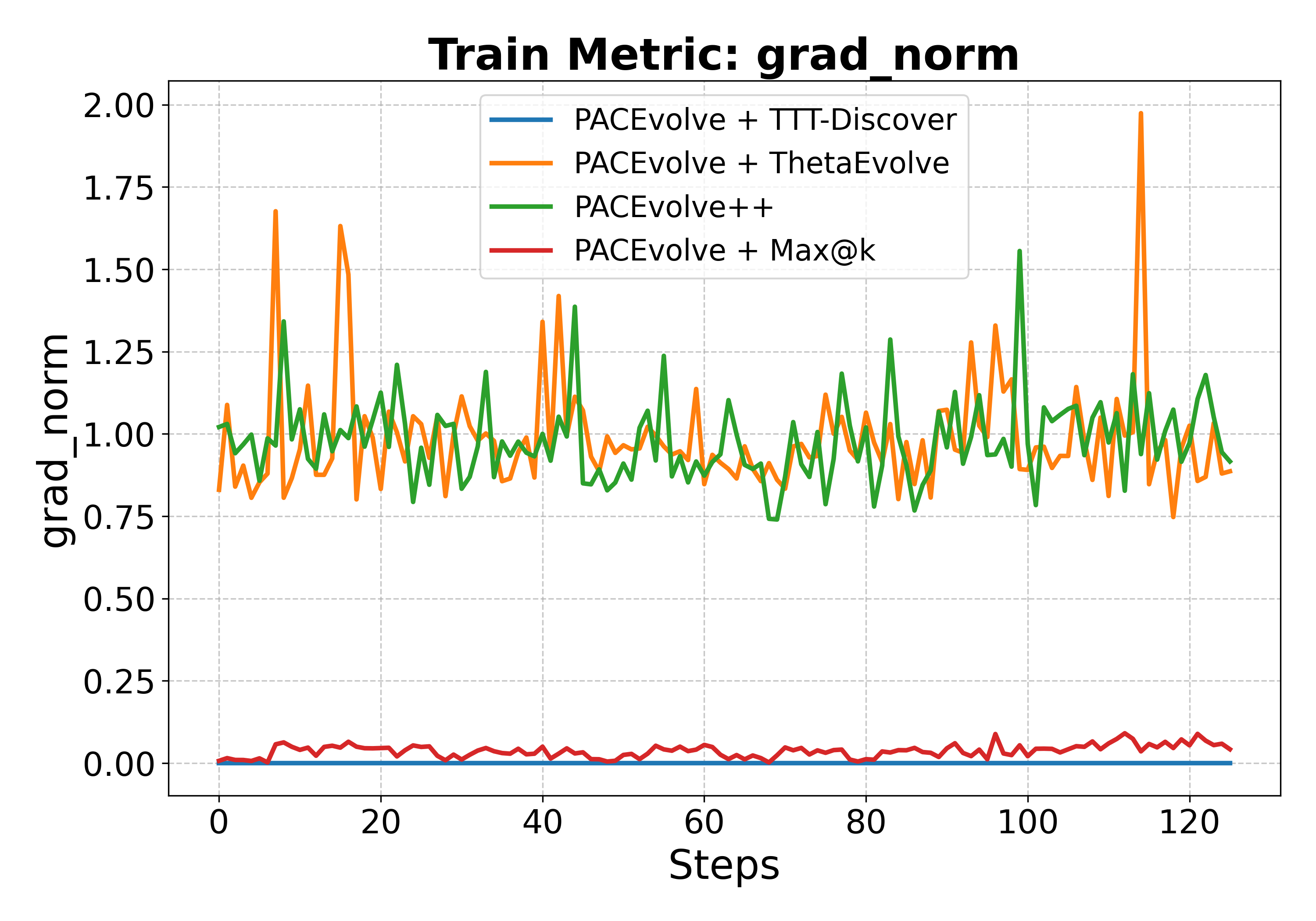}
        \caption{EPLB gradient norm.}
    \end{subfigure}
    \caption{EPLB training dynamics of 8B models. ThetaEvolve exhibits large gradient-norm spikes, Max@$k$ steadily collapses entropy, and TTT-Discover training collapsed quickly. \method{} remains comparatively well conditioned on both metrics.}
    \label{fig:diag_eplb_8b}
\end{figure*}

\begin{figure*}[t]
    \centering
    \begin{subfigure}[t]{0.49\textwidth}
        \centering
        \includegraphics[width=\linewidth]{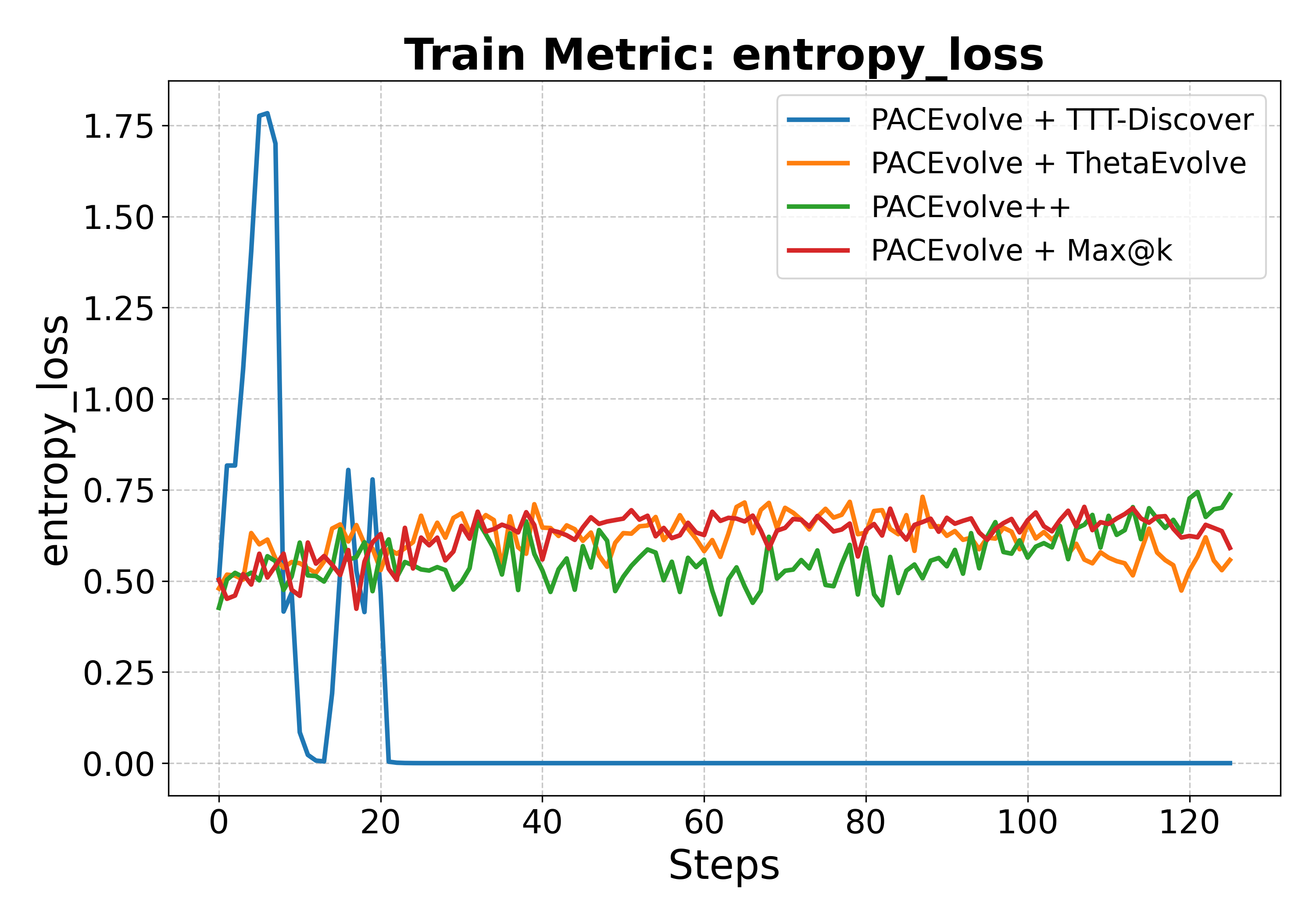}
        \caption{EPLB entropy.}
    \end{subfigure}
    \hfill
    \begin{subfigure}[t]{0.49\textwidth}
        \centering
        \includegraphics[width=\linewidth]{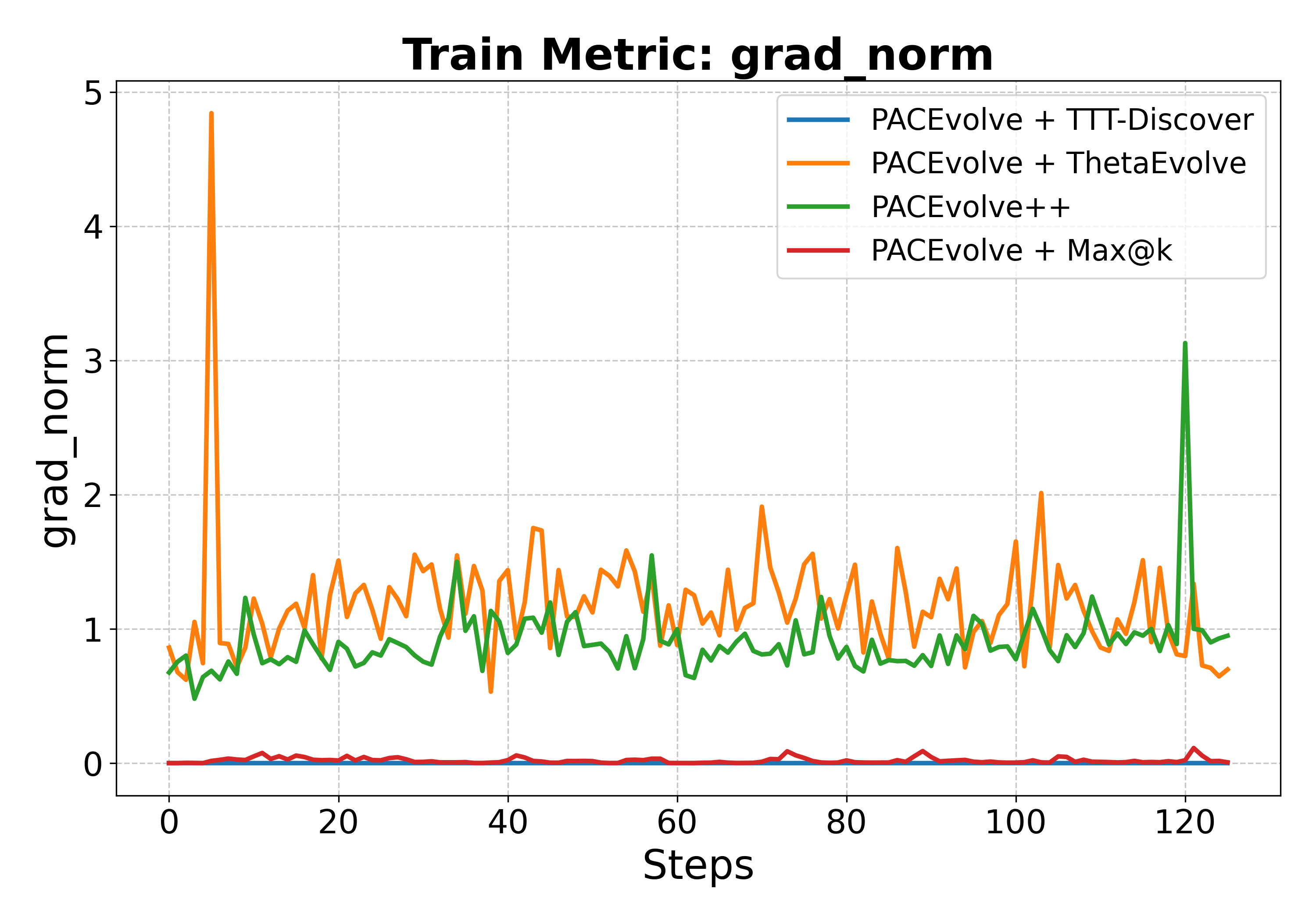}
        \caption{EPLB gradient norm.}
    \end{subfigure}
    \caption{EPLB training dynamics of 4B models. ThetaEvolve exhibits large gradient-norm spikes, Max@$k$ steadily collapses entropy, and TTT-Discover training collapsed quickly. \method{} remains comparatively well conditioned on both metrics.}
    \label{fig:diag_eplb_4b}
\end{figure*}

\clearpage

\subsection{Disaggregated Metrics}\label{appen:disaggre}

In this section, we report disaggregated metrics for the best candidate found by each RL variant. We omit the combined evolution score and the iteration at which the candidate was found. All values are rounded to three decimal places. We abbreviate DeepSeek-R1-0528-Qwen3-8B as DS-R1-Qwen3-8B in the tables. Method abbreviations are TTT-D for TTT-Discover, PACE++ for \method{}, and Max@$k$ for the PKPO objective.

\begin{table}[H]
\caption{Disaggregated EPLB metrics. Bal. denotes balancedness.}
\label{tab:eplb_disaggregated}
\centering
\begin{tabular}{lcccc}
\toprule
\multirow{2}{*}{Method} & \multicolumn{2}{c}{Qwen3.5-4B} & \multicolumn{2}{c}{DS-R1-Qwen3-8B} \\
\cmidrule(lr){2-3} \cmidrule(lr){4-5}
 & Bal. & Speed & Bal. & Speed \\
\midrule
TTT-D & 0.446 & 0.054 & 0.443 & 0.050 \\
GRPO & 0.443 & 0.055 & 0.478 & 0.056 \\
PACE++ & 0.489 & 0.053 & 0.510 & 0.035 \\
Max@$k$ & 0.407 & 0.040 & 0.458 & 0.080 \\
\bottomrule
\end{tabular}
\end{table}

\begin{table}[H]
\caption{Disaggregated KuaiRec metrics. N@10/50 denotes NDCG@10/50 and H@10/50 denotes HR@10/50.}
\label{tab:kuairec_disaggregated}
\centering
\resizebox{\textwidth}{!}{%
\begin{tabular}{lcccccccccc}
\toprule
\multirow{2}{*}{Method} & \multicolumn{5}{c}{Qwen3.5-4B} & \multicolumn{5}{c}{DS-R1-Qwen3-8B} \\
\cmidrule(lr){2-6} \cmidrule(lr){7-11}
 & N@10 & N@50 & H@10 & H@50 & MRR & N@10 & N@50 & H@10 & H@50 & MRR \\
\midrule
TTT-D & 0.073 & 0.101 & 0.116 & 0.250 & 0.069 & 0.103 & 0.134 & 0.155 & 0.302 & 0.097 \\
GRPO & 0.101 & 0.133 & 0.153 & 0.304 & 0.095 & 0.107 & 0.138 & 0.159 & 0.303 & 0.101 \\
PACE++ & 0.102 & 0.133 & 0.154 & 0.302 & 0.096 & 0.110 & 0.145 & 0.166 & 0.327 & 0.104 \\
Max@$k$ & 0.099 & 0.130 & 0.156 & 0.300 & 0.091 & 0.104 & 0.136 & 0.156 & 0.309 & 0.098 \\
\bottomrule
\end{tabular}%
}
\end{table}

\begin{table}[H]
\caption{Disaggregated Multi-Evolve metrics. P@5 denotes Precision@5.}
\label{tab:multievolve_disaggregated}
\centering
\begin{tabular}{lcccc}
\toprule
\multirow{2}{*}{Method} & \multicolumn{2}{c}{Qwen3.5-4B} & \multicolumn{2}{c}{DS-R1-Qwen3-8B} \\
\cmidrule(lr){2-3} \cmidrule(lr){4-5}
 & Pearson $r$ & P@5 & Pearson $r$ & P@5 \\
\midrule
TTT-D & 0.703 & 0.430 & 0.703 & 0.415 \\
GRPO & 0.675 & 0.449 & 0.688 & 0.456 \\
PACE++ & 0.726 & 0.440 & 0.733 & 0.452 \\
Max@$k$ & 0.630 & 0.430 & 0.697 & 0.448 \\
\bottomrule
\end{tabular}
\end{table}

These disaggregated results show that methods might be exploring different fronts along the Pareto-optimal curve. While a higher combined score generally signals stronger overall solutions, it does not guarantee Pareto dominance. For example, EPLB exposes a trade-off between balance and speed, KuaiRec separates short-horizon ranking quality from broader hit-rate coverage, and Multi-Evolve separates correlation from top-ranked mutant precision.

\clearpage

\section{Prompt templates} \label{appen:prompts}
The following are the prompt templates used for advisor and code implementation (Replace task-related information to deploy on other tasks). Texts in \textcolor{red}{red} represent task-specific placeholders; texts in \textcolor{blue}{blue} represent dynamic context managed during the evolutionary search.

\begin{promptbox}[Idea Generation Prompt Template]
\ttfamily
We are conducting an evolutionary optimization process for the Expert Parallelism Load Balancer (EPLB).

\textcolor{red}{{BACKGROUND}}

\textcolor{red}{{TASK INTRO}}

\textcolor{red}{{CODING REQ}}

Current state-of-the-art

The current state-of-the-art solution is as follows:

\textcolor{blue}{{SoTA Solution}}

Idea Repo

Idea repos contain the ideas we have generated so far and the experiments we have run to test these hypotheses.

\textcolor{blue}{{Idea Repo}}

Your Task

When proposing a new design, you should start by conducting a research brainstorming exercise to develop 3 options to explore the design space. Go through each option and provide a comprehensive explanation of the proposed changes.

Go through the idea and experiment history carefully; DO NOT re-propose an idea that has already been well tested.

You should follow the following format when generating ideas:

Idea 1 

Hypothesis: 

Reasoning:

Idea 2 

Hypothesis: 

Reasoning:

Idea 3 

Hypothesis: 

Reasoning: 
\end{promptbox}

\begin{promptbox}[Idea Selection Prompt Template]
\ttfamily
We are conducting an evolutionary optimization process for the Expert Parallelism Load Balancer (EPLB).

\textcolor{red}{{BACKGROUND}}

\textcolor{red}{{TASK INTRO}}

\textcolor{red}{{CODING REQ}}

Current state-of-the-art

The current state-of-the-art solution is as follows:

\textcolor{blue}{{SoTA Solution}}

Idea Repo

Idea repos contain the ideas we have generated so far and the experiments we have run to test these hypotheses.

\textcolor{blue}{{Idea Repo}}

Your Task

Your job is to come up with an experiment to test one of the ideas in the idea repo. Think about an experiment to run that will have the best shot of helping you accomplish your overall goal.

You should use the following format for the idea selection part:

Idea ID: 
Experiment description: 
\end{promptbox}

\begin{promptbox}[Code Implementation Prompt Template]
\ttfamily
We are conducting an evolutionary optimization process for the Expert Parallelism Load Balancer (EPLB).

\textcolor{red}{{BACKGROUND}}

\textcolor{red}{{TASK INTRO}}

\textcolor{red}{{CODING REQ}}

Current state-of-the-art

The current state-of-the-art solution is as follows:

\textcolor{blue}{{SoTA Solution}}

Your Task

Your job is to implement the selected idea above.

\textcolor{blue}{Idea ID: {Idea ID}}

\textcolor{blue}{Experiment description: {Experiment description}}

\textcolor{blue}{Selected idea: {Selected Idea}}
\end{promptbox}

\section{Training stability}
\label{appen:training_stability}
\newcommand{\R}{\mathbb{R}}

We analyze the scale-conditioned advantages of the hybrid objective when absolute reward differences compress, but relative ordering is preserved. This corresponds to late-stage evolution, where candidate solutions become local variants of already strong programs. The goal is to understand what standardization preserves: not an absolute reward scale, but a bounded credit-assignment geometry.

Let \(r_1,\ldots,r_N \in \R\) be a non-constant reward profile with mean
\(\bar r = \frac{1}{N}\sum_{i=1}^N r_i\) and standard deviation
\(\sigma_r > 0\). Assume strict ordering (i.e.,
\(r_i \neq r_j\) for \(i \neq j\)) and \(2 \leq k \leq N\).

For any offset \(c \in \R\) and scale \(\delta > 0\), define a scaled reward batch
\[
g_i^{(\delta)} = c + \delta r_i.
\]

This construction models \emph{reward compression}: as \(\delta\) becomes small, rewards become closer together while their rankings remain unchanged.

Then the ordering is preserved:
\[
r_i > r_j \iff g_i^{(\delta)} > g_j^{(\delta)}.
\]

Define the raw group-relative branch
\[
A_i^{G}(\delta)
=
g_i^{(\delta)} - \bar g^{(\delta)},
\]
and the SLOO\(_{k-1}\) weight
\[
w_i^{\mathrm{SLOO}}(\delta)
=
\frac{1}{\binom{N}{k}}
\sum_{\substack{I \subseteq \{1,\ldots,N\}\\ |I| = k,\ i \in I}}
\left(
\max_{j \in I} g_j^{(\delta)}
-
\max_{b \in I \setminus \{i\}} g_b^{(\delta)}
\right).
\]
For any non-constant branch vector \(B=(B_1,\ldots,B_N)\), define its scale-conditioned version
\[
\Phi_{\epsilon_{\mathrm{num}}}(B)_i
=
\frac{B_i-\mu(B)}{\sigma(B)+\epsilon_{\mathrm{num}}}.
\]

Then:

\begin{align}
A_i^{G}(\delta)
&=
\delta (r_i - \bar r), \\[4pt]
\|A^{G}(\delta)\|_2^2
&=
N \delta^2 \sigma_r^2, \\[6pt]
w_i^{\mathrm{SLOO}}(\delta)
&=
\delta\, w_i^{\mathrm{SLOO}}(1), \\[4pt]
\|w^{\mathrm{SLOO}}(\delta)\|_2^2
&=
\delta^2\, \|w^{\mathrm{SLOO}}(1)\|_2^2.
\end{align}

\begin{proof}[Proof of Theorem~\ref{thm:reward-collapse}]

We prove the result in three steps: first, deriving the scale-conditioned forms; second, showing boundedness; and finally, characterizing the SLOO credit geometry.

\medskip

We first analyze how the group-relative branch changes when rewards are scaled.

Since \(g_i^{(\delta)} = c + \delta r_i\), we compute:
\[
\bar g^{(\delta)} = c + \delta \bar r,
\qquad
\sigma(g^{(\delta)}) = \delta \sigma_r.
\]

We note that adding a constant shifts the mean but does not affect variance, while scaling by \(\delta\) scales the standard deviation linearly.

Substituting into the standardization operator gives
\[
\Phi_{\epsilon_{\mathrm{num}}}(A^G(\delta))_i
=
\frac{\delta (r_i - \bar r)}{\delta \sigma_r + \epsilon_{\mathrm{num}}}.
\]

Taking the squared norm:
\[
\|\Phi_{\epsilon_{\mathrm{num}}}(A^G(\delta))\|_2^2
=
\frac{N \delta^2 \sigma_r^2}{(\delta \sigma_r + \epsilon_{\mathrm{num}})^2}.
\]

For the raw group-relative signal, since it only centers the rewards, we obtain directly:
\[
A_i^{G}(\delta)
=
g_i^{(\delta)} - \bar g^{(\delta)}
=
\delta (r_i - \bar r),
\]
and thus
\[
\|A^{G}(\delta)\|_2^2
=
N \delta^2 \sigma_r^2.
\]

\medskip

We now analyze how the SLOO estimator behaves under the same transformation.

The key observation is that the \(\max\) operator has two properties:

1. \emph{Translation equivariance:} \(\max(c + x_i) = c + \max(x_i)\)

2. \emph{Positive homogeneity:} \(\max(\delta x_i) = \delta \max(x_i)\) for \(\delta > 0\)

Applying these to any subset \(I\), we obtain:
\[
\max_{j \in I} g_j^{(\delta)}
=
c + \delta \max_{j \in I} r_j,
\]
\[
\max_{b \in I \setminus \{i\}} g_b^{(\delta)}
=
c + \delta \max_{b \in I \setminus \{i\}} r_b.
\]

Subtracting, the constant \(c\) cancels:
\[
\max_{j \in I} g_j^{(\delta)}
-
\max_{b \in I \setminus \{i\}} g_b^{(\delta)}
=
\delta
\left(
\max_{j \in I} r_j
-
\max_{b \in I \setminus \{i\}} r_b
\right).
\]

Therefore, SLOO depends on winner-changing margins, and its raw signal scales linearly with \(\delta\).

Averaging over subsets yields
\[
w_i^{\mathrm{SLOO}}(\delta)
=
\delta\, w_i^{\mathrm{SLOO}}(1),
\]
and thus
\[
\|w^{\mathrm{SLOO}}(\delta)\|_2^2
=
\delta^2\, \|w^{\mathrm{SLOO}}(1)\|_2^2.
\]
Applying \(\Phi_{\epsilon_{\mathrm{num}}}\) gives
\[
\Phi_{\epsilon_{\mathrm{num}}}(w^{\mathrm{SLOO}}(\delta))_i
=
\frac{\delta\left(w_i^{\mathrm{SLOO}}(1)-\mu(w^{\mathrm{SLOO}}(1))\right)}
{\delta\sigma(w^{\mathrm{SLOO}}(1))+\epsilon_{\mathrm{num}}}.
\]

\medskip

We now show boundedness. For any non-constant branch \(B\),
\[
\|\Phi_{\epsilon_{\mathrm{num}}}(B)\|_2^2
=
\frac{N\sigma(B)^2}{(\sigma(B)+\epsilon_{\mathrm{num}})^2}
\leq N.
\]
This applies to both \(A^G(\delta)\) and \(w^{\mathrm{SLOO}}(\delta)\). Moreover, if \(B_i^{(\delta)}=\delta B_i^{(1)}\), then \(\Phi_{\epsilon_{\mathrm{num}}}(B^{(\delta)})\) approaches the z-score vector \((B_i^{(1)}-\mu(B^{(1)}))/\sigma(B^{(1)})\) whenever \(\delta\sigma(B^{(1)})\gg\epsilon_{\mathrm{num}}\), and approaches zero whenever \(\delta\sigma(B^{(1)})\ll\epsilon_{\mathrm{num}}\). Thus, the scale-conditioned branch is bounded in the diverse regime and naturally becomes uninformative in the collapsed regime, where our implementation skips the update.

\medskip

It remains to characterize what the SLOO branch preserves after scale conditioning. For any subset \(I\) containing \(i\), the term
\[
\max_{j\in I} r_j - \max_{b\in I\setminus\{i\}} r_b
\]
is positive if and only if \(i\) is the highest-reward element in \(I\). Otherwise, removing \(i\) does not change the subset maximum, and the term is zero. If responses are ranked in decreasing reward order and \(i\) has rank \(m\), then \(i\) can be the winner only in subsets whose other \(k-1\) elements are drawn from the \(N-m\) lower-ranked responses. Therefore, the bottom \(k-1\) responses cannot win any size \(k\) subset and receive zero raw SLOO contribution. Standardization is an affine transform with positive scale, so it preserves the ordering induced by these SLOO frontier-contribution scores.

\medskip

The theorem follows. The group-relative branch provides dense, centered reward credit, which is useful when early-stage rollout groups are diverse. The SLOO branch gives frontier-contribution credit, retaining the same ordering of candidates after scale conditioning and aligning better with late-stage best-of-\(k\) evolutionary survival.

\end{proof}


\end{document}